\documentclass[a4paper,11pt]{article}
\usepackage[
    bottom=4.2cm,
    top=4cm,
    left=3.6cm,
    right=3.6cm]{geometry}
\usepackage{graphicx}
\usepackage{amssymb}
\usepackage{amsmath}
\usepackage{mathtools}
\usepackage{microtype}
\usepackage{amsthm}
\usepackage[mathscr]{euscript}
\usepackage{csquotes}
\usepackage[ruled]{algorithm2e}
\usepackage[T1]{fontenc}
\usepackage{makecell}
\usepackage{dsfont}
\usepackage{hyphenat}
\usepackage{interval}
\usepackage{hyperref}
\usepackage[title]{appendix}
\usepackage{authblk}
\usepackage[table]{xcolor}
\hypersetup{
    colorlinks,
    linkcolor={red!50!black},
    citecolor={blue!50!black},
    urlcolor={blue!80!black}
}
\usepackage[english]{babel}

\makeatletter
\renewenvironment{algomathdisplay}
 {\[}
 {\@endalgocfline\vspace{-\baselineskip}\]{\DontPrintSemicolon\;}}
\makeatother

\usepackage[style=apa, backend=biber]{biblatex}
\newtheorem{theorem}{Theorem}

\newtheorem{lemma}{Lemma}
\newtheorem{assumption}{Assumption}
\newtheorem{proposition}{Proposition}
\newtheorem{corollary}{Corollary}

\DeclareMathOperator*{\argmax}{argmax}
\DeclareMathOperator*{\argmin}{argmin}
\DeclareMathOperator{\Span}{span}

\title{Parallel gradient boosting for flexible estimation of conditional distributions}
\author[1,2,3]{Rémy Chapelle}
\author[2]{Nicolas Vayatis}
\author[1]{Bruno Falissard}
\author[1]{Mohammed Sedki}
\affil[1]{\normalsize Université Paris-Saclay, UVSQ, Inserm, CESP, 94807, Villejuif, France}
\affil[2]{Université Paris-Saclay, Université Paris Cité, ENS Paris-Saclay, CNRS, SSA, Inserm, Centre Borelli, 91190, Gif-sur-Yvette, France}
\affil[3]{\'Ecole du Val-de-Grâce, Service de Santé des Armées, 75005, Paris, France}

\date{}

\begin{document}

\maketitle

\begin{abstract}
Boosting is one of the most successful learning techniques for standard classification and regression tasks. Its extension to multi-output prediction problems has found an increasing number of applications in recent years. Among them is the prediction of entire conditional distributions rather than single functionals, which can often be framed as a multi-output regression problem, for example multiple quantile regression. Addressing such problems with classical implementations of boosting is computationally challenging, because usually one base model is trained for each target at every iteration. More efficient variants of boosting have been proposed to speed up training, but they tend to be tied to specific loss functions and classes of base learners, usually decision trees. In this work, we study a modification of the gradient boosting algorithm, which we call parallel gradient boosting, designed to circumvent all these limitations. The core idea is to use a common descent direction for all training observations. By doing so, only one base model is needed at each iteration, regardless of the number of targets, which allows for considerable performance gains. We establish sufficient conditions for the convergence of the algorithm, whose practical use is introduced via the multiple quantile regression setting. We show that in such a setting, it provides predictions of similar quality to state-of-the-art boosting libraries such as XGBoost, while being faster by several orders of magnitude. Then, we evaluate the properties of the resulting conditional distribution estimator, which is shown empirically to outperform other nonparametric and semiparametric estimators, especially in high-dimensional settings and in the presence of mixed and/or missing covariates.
\end{abstract}

\newpage
\section{Introduction}

The classical concern of supervised statistical learning is to determine properties of the conditional distribution of a dependent variable given some covariates. In the case of numeric labels, the property of interest is often the conditional mean, and many regression methods effectively provide estimates of this parameter. 

In modern applications of machine learning, estimating only some location parameters can be unsatisfactory, and one would prefer modeling the whole conditional distribution. This allows for more versatile analyses and is especially valuable in fields where the distribution is potentially complex and/or multimodal, so that the sufficiency of such parameters is not guaranteed (see Figure~\ref{fig:dens}). For example, in healthcare, the tails of the distribution are often of equal importance to its expected value, as they may identify individuals with extreme outcomes \parencite{jones2015healthcare, strobl2021dirac}. Other applications of modeling entire conditional distributions include probabilistic forecasting \parencite{krzysztofowicz2024probabilistic}, imputation of missing values \parencite{enders2025missing}, and data synthesis \parencite{cormode2025synthetic}. Distributional estimates can also be used to obtain prediction intervals, potentially with finite-sample guarantees for (conditional) coverage when done under the conformal prediction framework \parencite{zhou2025conformal, plassier2025probabilistic, zou2026conformalized}.

\begin{figure}[!htbp]
\centering
\includegraphics[width=0.7\linewidth]{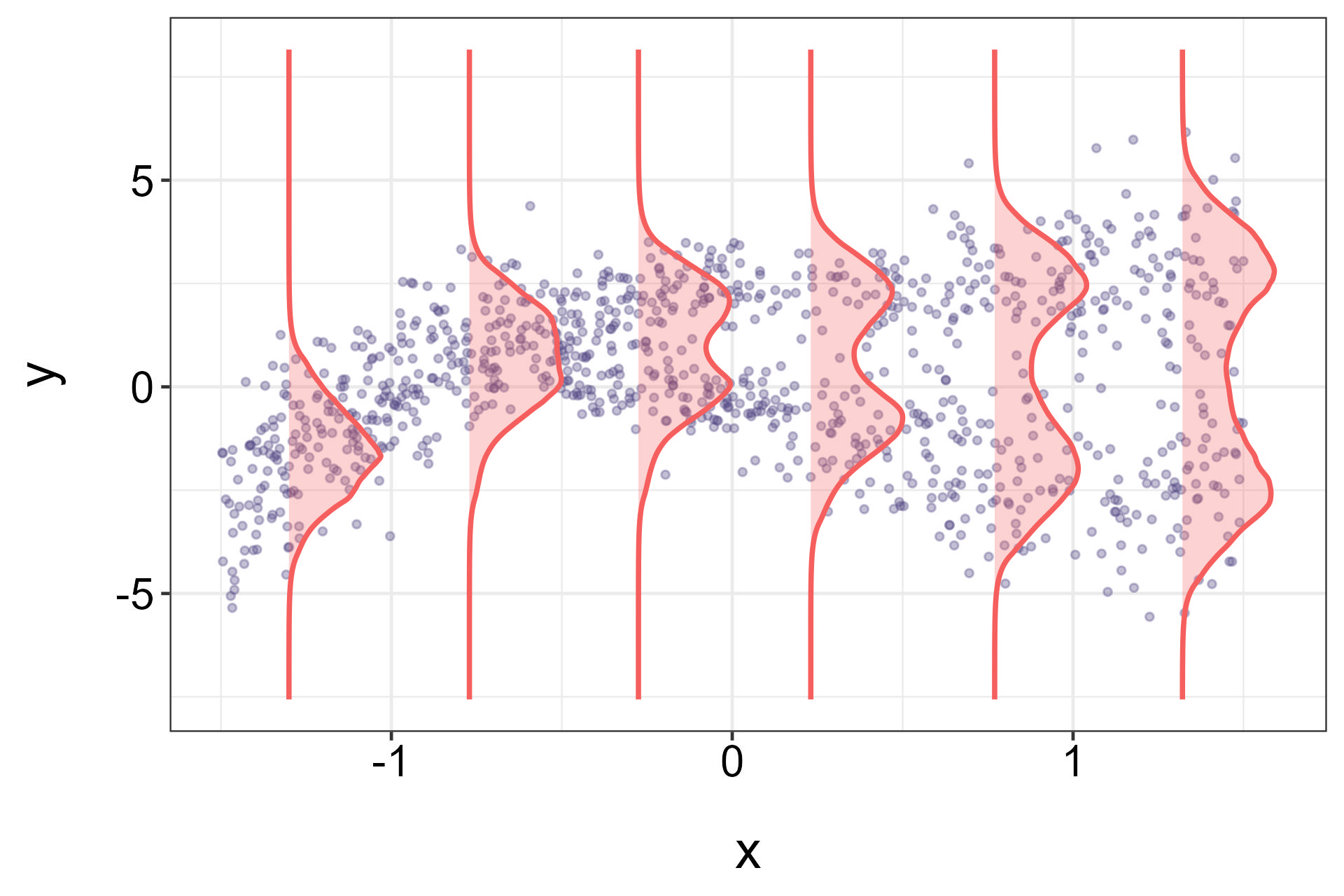}
\caption{\label{fig:dens}Estimation of conditional distributions which are bimodal for some values of the conditioning variable. The data originate from \textcite{lei2014distribution}. The estimates were obtained using the parallel gradient boosting method presented in this paper.}
\end{figure}

Modeling conditional distributions is straightforward if one assumes a parametric family. This indeed amounts to estimating a finite number of parameters, which can be done using traditional modeling techniques. In real-life applications, however, it is often difficult to specify such parametric forms, and one could prefer working with nonparametric models \parencite{hill1982robustness}. By adapting to the complexity of the data at hand, these models and the corresponding algorithms can prove more robust and produce more reliable estimates than their parametric counterparts \parencite{leech2002call}.

In the case of categorical labels, nonparametric prediction of class probabilities is already standard practice. It is at the core of major classification algorithms, such as the naive Bayes classifier \parencite{oro96521}, and the output of many other algorithms can be interpreted similarly via their training loss function \parencite{silva2023classifier}. Unfortunately, these approaches are not applicable as-is to continuous labels, especially because in this case the observed data do not cover the entire support of the variable of interest.  To address this difficulty, one intuitive strategy is to frame the corresponding learning problem as multiple regression ones (or, equivalently, as a multi-output regression one). The idea is then to base the estimate of the conditional distribution on estimates of several functionals of this distribution, for example, conditional quantiles \parencite{qiu2026review}. In any case, the number of targets must grow with the training data for the resulting estimator to remain nonparametric, which poses computational challenges.

If computation times were not an issue, boosted regression methods would make appealing candidates to perform such tasks \parencite{bentejac2021comparative}. They are based on iteratively combining weak learners (usually, trees) into a strong one, hence reducing the bias of the corresponding estimator \parencite{schapire2013boosting}. Gradient boosting is remarkably effective for real-life prediction tasks, and has been the winner of many machine learning competitions since its introduction \parencite{Breiman:96:TR,nielsen2016tree}. Its popularity grew even more in the last 10 years with the release of high-performance implementations such as XGBoost \parencite{chen2016xgboost}, LightGBM \parencite{ke2017lightgbm} and CatBoost \parencite{prokhorenkova2018catboost}. However, these libraries do not provide both scalable and flexible solutions to perform multi-output regressions. More precisely, multi-output models returned either involve several single-output weak learners, or one multi-output weak learner per iteration. In the latter case, the class of base models used is tied to a specific loss, and generally makes the biased assumption that all targets are fully correlated \parencite{joly2019gradient}. This makes these algorithms poorly suited to conditional distribution estimation using the approaches mentioned above.

\paragraph{Contributions.}
In this paper, we propose an efficient extension of the classical gradient boosting algorithm to the multi-output prediction setting. This approach, which we name parallel gradient boosting (PGB), is based on specifying a common direction of descent for all observations at the current iteration. By doing so, PGB requires training only one base model per iteration, which can be any univariate prediction rule, and naturally exploits any correlation between targets. Our formulation of PGB generalizes several strategies mentioned in the literature on decision tree boosting, which had not been studied on their own, nor applied to other classes of base learners. For separable loss functions, we propose an implementation of PGB with proved convergence in empirical risk to the minimum attainable with the chosen family of base models. We empirically evaluate the behavior of the algorithm on multiple quantile regression tasks and its extension to conditional density estimation. Our results show that PGB provides fast and reliable estimation of conditional distributions in settings traditionally considered challenging. These include high-dimensional data, irrelevant covariates, mixed-type variables, and the presence of missing values. An efficient implementation of the algorithm is provided within the R package \texttt{opencesp} available on CRAN \parencite{opencesp2026}.

\paragraph{Organization of this paper.}
The rest of the paper is organized as follows. In Section~\ref{sec:related}, we introduce our setting and discuss the connections of our contribution with the existing literature. In Section~\ref{sec:gbm}, we recall the general formulation of multi-output gradient boosting, before describing the PGB procedure in Section~\ref{sec:pgb}. We then describe our conditional distribution estimator based on PGB in Section~\ref{sec:cde}. In Section~\ref{sec:experiments}, we evaluate the empirical performance of the PGB procedure in several experiments involving simulated and real data. We conclude with a discussion of our findings in Section~\ref{sec:discussion}. All theoretical proofs are deferred to the appendix.

\section{Background}\label{sec:related}

In this section, we introduce the setting and objectives of PGB before briefly reviewing the related literature.

\subsection{Problem formulation}

Let $(X,Y)$ be a pair of random variables taking values in $\mathcal{X}\times\mathbb{R}$, with $\mathcal{X}$ being any $d$-dimensional measurable feature set. We suppose that we are given $n$ independent realizations $\mathscr{D}_n=\big((\boldsymbol{x}_1,y_1),\ldots,(\boldsymbol{x}_n,y_n)\big)$ of $(X,Y)$. In this work, we are interested in using the sample $\mathscr{D}_n$ to estimate the conditional distribution of $Y$ given $X$. We require the estimator to be nonparametric and applicable to mixed-type and potentially high-dimensional covariates. Our approach is to frame this problem of conditional distribution estimation as a problem of multiple quantile regression, which PGB is designed to handle in an algorithmically efficient way. Although the techniques employed in PGB are especially well-suited to multiple quantile regression, they are, in theory, applicable to other multi-output learning tasks. PGB therefore lies at the intersection of several rich bodies of literature, which we briefly review below.

\subsection{Related works}

Some attempts have been reported in the literature to use gradient boosting for probabilistic predictions in a given parametric family \parencite{duan2020ngboost, marz2019xgboostlss, marz2022distributional}. We do not discuss these works in detail, as the nonparametric nature of PGB makes it primarily compete with other nonparametric estimators of conditional distributions. Such approaches can be classified according to many criteria. Here we consider a distinction between instance-based approaches and what we call (following \cite{reisach2025transforming}) \enquote{regression-based} approaches. This distinction is only for expositional purposes, as some methods could qualify in both categories depending on how they are formulated. We give an overview of these two types of approaches (a more comprehensive review is provided e.g. in \cite[Chapter~3]{Izbicki2025}), before examining the literature on multi-output regression and multi-task learning, including methods based on gradient boosting.

\paragraph{Instance-based estimation of conditional distributions.}
Apart from direct approaches based on $k$-nearest neighbors \parencite{benezet2025learning}, the most classical nonparametric conditional density estimator is probably the kernel estimator \parencite{rosenblatt1969conditional}. It is a straightforward extension of marginal kernel density estimation to the estimation of conditional densities. The principle is to express the conditional density of $Y$ given $X$ in terms of their joint density and the marginal density of $X$, and to estimate both using the classical kernel density estimator. Unfortunately, the approach suffers the same limitation of the method it extends: it is only applicable to continuous data types (although some attempts have been made to extend it to mixed-type settings, see e.g. \cite{racine2004kernel}), and notoriously fails in high dimension due to the curse of dimensionality \parencite{geenens2011curse}. Additionally, the choice of the hyperparameters (i.e. the kernel functions and associated bandwidths) can be tedious in practice \parencite{zhao2025adaptive}. Distributional random forests \parencite{cevid2022distributional} are a related approach that relies on specialized binary trees to obtain a weighting function on the set of covariates. By doing so, it inherits the desirable properties of supervised random forests \parencite{scornet2026theory}, namely their flexibility and ease of hyperparameter tuning. Other, more marginal approaches are based on histograms \parencite{sart2017estimating, yang2024conditional}. They share with the kernel methods described above a dependence on an explicit (rather than learned) geometry on the sample space, and therefore face similar limitations in high dimension. Another line of work that is close to these instance-based approaches directly estimates the conditional density through least-squares density-ratio estimation \parencite{sugiyama2010conditional}.

\paragraph{Regression-based estimation of conditional distributions.}
Most other methods for estimating conditional distributions rely on estimating several functionals of these conditional distributions, i.e., performing multiple regressions. For these methods to qualify as nonparametric, the number of estimated functionals must grow with the size of the data (either explicitly or through hyperparameter optimization). One straightforward application of this idea is to estimate $P(Y\leq y_{\tau}\mid \boldsymbol{x})$ for several cutoffs $y_{\tau}\in\mathbb{R}$, which is equivalent (by the definition of conditional expectations) to regressing $\mathds{1}\{Y\leq y_{\tau}\}$ on $X$, where $\mathds{1}$ represents the indicator function. This is the \enquote{distributional regression} approach, described in \textcite{hall1999methods, koenker2013distributional}. Another related approach is based on quantile regression \parencite{verbois2018probabilistic,vasseur2021comparing, tyralis2021explanation, koenker2005quantile, gneiting2007strictly}, where this time the conditional distribution is described through the corresponding conditional quantile function, estimated at several thresholds, each in $\interval[open]{0}{1}$. The estimation can be performed by gradient boosting, in which case limiting computation times can prove challenging \parencite{taieb2016forecasting, landry2016probabilistic, papacharalampous2022probabilistic}. More sophisticated regression-based approaches include FlexCode \parencite{Izbicki2017-lo}, in which conditional probability density functions (p.d.f.'s) are represented by a truncated orthogonal series expansion, and the expansion coefficients are estimated by regression. Other nonparametric approaches that are compatible with a broad range of regression methods (including gradient boosting) rely on pooled hazard regressions \parencite{diaz2011super}, Poisson regressions on discretized targets using Lindsey's method \parencite{gao2022lincde}, or appropriately generated auxiliary samples \parencite{reisach2025transforming}. Methods stemming from generative modeling, such as conditional diffusion models \parencite{beltran2024treeffuser}, also fall into this category, although they typically do not provide explicit estimates of the conditional distribution (instead, they provide a way of approximately sampling from this distribution). Neural networks have also found applications for the estimation of conditional distributions, either via such generative approaches \parencite{mirza2014conditional, liu2025conditional, melnychuk2023normalizing}, or through more classical parameter estimation \parencite{rothfuss2019conditional, xu2017composite}.

\paragraph{Multi-output regression.}
Multi-output regression algorithms simultaneously learn prediction rules for several numeric variables, with the aim of improving both predictive performance and computational efficiency  by exploiting their relatedness \parencite{borchani2015survey}. This ambitious goal has led to several proposals of both parametric \parencite{Kim2012rn, li2017better, chen2013convex} and nonparametric nature. The latter have especially focused on tree methods, in part because existing algorithms can be easily modified to support multiple targets. This is done by considering suitable splitting criteria \parencite{rahman2017integratedmrf, d2017regression, simm2014tree}, as well as by adapting prediction rules at terminal nodes \parencite{jeong2020regularization}. Although such algorithms outperform their univariate counterparts in terms of computation times \parencite{schmid2023tree}, they are tied to a specific loss, and often make the implicit assumption that all targets are fully correlated, which introduces bias \parencite{joly2019gradient}. These limitations also affect derived algorithms based on ensemble methods, such as random forests \parencite{kocev2013tree} and gradient boosting \parencite{liu2019hitboost, emami2025condensed}. Another way to use ensemble methods for multi-output regression is to assign to each base learner an appropriate univariate objective, chosen so that solving the (potentially infinite) set of such univariate problems is equivalent to solving the multivariate one. This is the approach followed in \textcite{joly2019gradient}, following \textcite{joly2014random}, where univariate targets are obtained by random projections of the output space. The intuition is close to ours, but is tied to decision trees, and the authors do not study their approach from an optimization point of view. For example, they do not establish how to choose the projection matrices to minimize the empirical risk. Similar ideas are also explored (among others) by \textcite{iosipoi2022sketchboost} in a framework named SketchBoost, with the same limitations. As such, PGB can be considered a generalization to arbitrary classes of base learners of a particular configuration of the works of \textcite{joly2019gradient, iosipoi2022sketchboost} with desirable properties (especially, convergence). Finally, other proposed approaches to multi-output regression include those based on support vector regression \parencite{vazquez2003multi, tran2024critical}, kernels \parencite{baldassarre2012multi, alvarez2012kernels, arashloo2022multi}, neural networks \parencite{emami2024deep}, as well as the reformulation of the multivariate problem into several univariate ones, which are then often treated separately \parencite{au2019component, wang2025semantics}.

\paragraph{Multi-task learning.}
Multi-task learning aims to solve several learning tasks jointly by leveraging their similarities \parencite{argyriou2006multi, zhang2021survey}. This is a wide subfield of machine learning, which includes multi-output regression as a particular case \parencite{thung2018brief}. However, contrary to the classical setting in multi-output regression, in multi-task learning, each of the individual problems can involve a different feature space and target variables \parencite{xu2019survey}. The standard literature on multi-task learning is organized by the choices made to model how the tasks relate to one another. A historical and technical review of this literature is provided by \textcite{yu2025multitask, liu2025multi, yu2025multitaskb}. Briefly, one of the first proposals was to make neural-networks share hidden representations \parencite{caruana1997multitask}; regularized methods then introduced shared and task-specific parameters \parencite{evgeniou2004regularized}, sparse common feature maps \parencite{argyriou2006multi}, low-rank structures, task covariance matrices, or task clusters \parencite{jacob2008clustered, zhang2018overview}. Recent works pursue these research directions, and also study the balancing of task losses or gradients to attain more homogeneous learning trajectories across tasks, hence improving the quality of the resulting model. A popular example of such methods for deep neural networks is called Gradnorm \parencite{chen2018gradnorm}; more recent alternatives include PCGrad \parencite{yu2020gradient} and CAGrad \parencite{liu2021conflict}. Boosting has been less central in this literature, but several constructions have been proposed: AdaBoost variants with multi-task weak learners \parencite{faddoul2010boosting}, boosted decision trees combined into one common as well as several task-specific prediction rules \parencite{chapelle2011boosted, emami2023multi}, boosted multi-output decision trees with task-dependent learning rates \parencite{ying2022mt}, and methods that use gradient boosting as a pre-processing step to determine clusters of related tasks \parencite{emami2026robust}. Another notable contribution is that of \textcite{zhang2012multi}, which explores the same space of predictors as in PGB, i.e., linear combinations of scalar-valued functions with vector coefficients. However, in this approach, the coefficients are viewed as a means of regularization rather than optimization, which moves the algorithm beyond simple gradient descent in a function space, and incurs additional algorithmic steps and costs.

\section{Gradient boosting for multi-output regression}\label{sec:gbm}

In this section, we recall the classical formulation of gradient boosting for multi-output regression, from which PGB derives. We consider a differentiable loss function\footnote{In practice, the gradient boosting procedure has also been successfully applied to loss functions that are not everywhere differentiable, such as the $L_1$ loss \parencite{natekin2013gradient}.} $L:\mathbb{R}\times\mathbb{R}^M\rightarrow\mathbb{R}^+$, and, for every predictor $g:\mathcal{X}\rightarrow\mathbb{R}^M$, the corresponding unnormalized empirical risk:
\begin{equation*}
\hat{\mathcal{R}}_n(g(\boldsymbol{x}_1),\ldots,g(\boldsymbol{x}_n))=\sum_{i=1}^nL(y_i,g(\boldsymbol{x}_i)).
\end{equation*}
In what follows, we will use slightly overloaded notations, namely $L_i(g)$ instead of $L(y_i,g(\boldsymbol{x}_i))$ and $\hat{\mathcal{R}}_n(g)$ instead of $\hat{\mathcal{R}}_n(g(\boldsymbol{x}_1),\ldots,g(\boldsymbol{x}_n))$. Keep in mind, however, that these are neither functionals nor functions of parameter values. Our aim is to find a predictor $g$ minimizing the expected risk $\mathbb{E}[L(Y,g(X))]$. For that, we are given a parametric family of base models $(\varphi_{\boldsymbol{\theta}})_{\boldsymbol{\theta}\in\Theta}$, where $\Theta$ is some parameter space, and each $\varphi_{\boldsymbol{\theta}}$ is a function from $\mathcal{X}$ to $\mathbb{R}^M$. In most implementations of single-output gradient boosting, default base models are limited-depth decision trees. In typical existing implementations of multi-output gradient boosting (see e.g. \cite{chen2016xgboost}), each model is formed by $M$ trees. Boosting procedures explore the space of predictors of the following form:
\begin{equation}\label{eq:gbm-form}
g(\boldsymbol{x})=\sum_{t=0}^T b_t\varphi_{\boldsymbol{\theta}_t}(\boldsymbol{x})
\end{equation}
That is, they explore the linear span of the set of base functions, which we denote (slightly abusing notations) by $\Span(\varphi_{\boldsymbol{\theta}})$. The exploration is performed iteratively, producing a sequence of predictors $(g_0,\ldots,g_T)$ up to some $T\geq1$. Before the first iteration, the model is a constant. To simplify the theoretical discussion, we choose to set $g_0(\boldsymbol{x})=0$. At iteration $t\geq 1$, the algorithm requires evaluating the gradient of $L(y_i,\cdot)$ at $g_{t-1}(\boldsymbol{x}_i)$ for $i\in\{1,\ldots,n\}$, which we denote by $\nabla_2L_i(g_{t-1})$ (making the same type of notational overload as above, and where the subscript in $\nabla_2$ indicates that the gradient is with respect to the second argument of $L$). In the case of multi-output regression, each of the $n$ gradient values lies in $\mathbb{R}^M$. To update the prediction function (i.e. to obtain $g_t$), we compute $\boldsymbol{\theta}_t$ such that:
\begin{equation}\label{eq:paramit}
   \boldsymbol{\theta}_t\in\argmin_{\boldsymbol{\theta}\in\Theta} \sum_{i=1}^n\langle\nabla_2L_i(g_{t-1}),\varphi_{\boldsymbol{\theta}}(\boldsymbol{x}_i)\rangle.
\end{equation}
That is, we look for parameter values minimizing the empirical inner product with the gradient of $L$ (we assume throughout the paper that such values always exist). The motivation for this choice of parameters comes from a first-order approximation of the empirical risk function. For any $b\in\mathbb{R}$ and $\boldsymbol{\theta}\in\Theta$, we write the following first order Taylor approximation:
\begin{align}\label{eq:firstorder}
\sum_{i=1}^nL_i(g_{t-1}+b\varphi_{\boldsymbol{\theta}})&\approx \sum_{i=1}^n(L_i(g_{t-1})+\langle\nabla_2L_i(g_{t-1}),b\varphi_{\boldsymbol{\theta}}(\boldsymbol{x}_i)\rangle) \\
&=\sum_{i=1}^nL_i(g_{t-1})+b\sum_{i=1}^n\langle\nabla_2L_i(g_{t-1}),\varphi_{\boldsymbol{\theta}}(\boldsymbol{x}_i)\rangle. \nonumber
\end{align}
Hence, solving Equation~\eqref{eq:paramit} provides a way of approximately minimizing the left-hand side of Equation~\eqref{eq:firstorder}. The latter provides the maximum decrease in empirical risk attainable at iteration $t$. The assumption that motivates practical implementations of gradient boosting is that Equation~\eqref{eq:paramit} can be solved easily: this is the \textit{multi-output weak learner} assumption. We then set:
\begin{equation}\label{eq:itboost}
    g_t=g_{t-1}+b_t\varphi_{\boldsymbol{\theta}_t}
\end{equation}
for some appropriately chosen $b_t\in\mathbb{R}$. After $T$ iterations, the predictor $g_T$ is of the form given in Equation~\eqref{eq:gbm-form}.

The learning procedure for $M$-dimensional predictions is therefore very similar to that for unidimensional predictions, which has been described in detail in the literature \parencite{schapire2013boosting}. To analyze the convergence of the algorithm, we may rely on the following assumptions.

\begin{assumption}[Smoothness of the loss]\label{assum:smooth}
    The loss $L$ is convex and $K$-smooth in its second argument.
\end{assumption}

\begin{assumption}[Symmetry of base models]\label{assum:sym}
    For any $\boldsymbol{\theta}\in\Theta$, there is a $\bar{\boldsymbol{\theta}}\in\Theta$ such that $\varphi_{\boldsymbol{\theta}}=-\varphi_{\bar{\boldsymbol{\theta}}}$.
\end{assumption}

\begin{assumption}[Boundedness of base models]\label{assum:bound}
    For any $\boldsymbol{\theta}\in\Theta$, $\Vert\varphi_{\boldsymbol{\theta}}(\boldsymbol{x}_i)\Vert_2^2\leq\frac{1}{n}$.
\end{assumption}

\begin{assumption}[Compactness of predictions]\label{assum:compact}
    The set
    \begin{equation*}
    \Gamma=\{(g(\boldsymbol{x}_1),\ldots,g(\boldsymbol{x}_n)) \mid g\in\Span(\varphi_{\boldsymbol{\theta}}),\hat{\mathcal{R}}_n(g)\leq\hat{\mathcal{R}}_n(g_0)\}
    \end{equation*}
    is compact.
\end{assumption}

Assumptions~\ref{assum:smooth} and \ref{assum:sym} are standard in the literature on gradient boosting (see \cite{biau2021optimization}). Assumption~\ref{assum:bound} is only for expositional convenience in the proofs of convergence, as it avoids the need for explicit upper bounds. Assumption~\ref{assum:compact} is useful to relate the empirical risk with the alignment between the gradient values and the predictions. It is actually satisfied as long as the set $\Gamma$ is bounded (to see that, remark that $\Gamma$ can be defined as an intersection of two closed sets, and that in $\mathbb{R}^{nM}$, compactness is equivalent to closedness and boundedness by the Borel-Lebesgue theorem). This is, for example, the case when $L$ is a sum of $M$ quadratic loss functions, one for each component of the prediction. Note that Assumption~\ref{assum:compact} immediately implies that $\argmin_{g\in\Span(\varphi_{\boldsymbol{\theta}})}\hat{\mathcal{R}}_n(g)$ is non-empty.

Under these assumptions, one can prove that multi-output gradient boosting converges in empirical risk for a suitable choice of step sizes, which would also be central to establishing its consistency (see e.g. \cite{lugosi2004bayes}, and \cite{bartlett2006adaboost}). The result is stated more formally in Theorems~\ref{th:gber} and \ref{th:gber2}, of which proofs are given in Appendix~\ref{sec:proofs}.

\begin{theorem}[Convergence of multi-output gradient boosting with decreasing step sizes]\label{th:gber}
Suppose that Assumptions~\ref{assum:smooth} to \ref{assum:compact} are fulfilled and let $g_\star\in\argmin_{g\in\Span(\varphi_{\boldsymbol{\theta}})}\hat{\mathcal{R}}_n(g)$. Let the step size at iteration $t\geq 1$ be chosen as:
\begin{equation*}
b_t=\min\{\frac{1}{\sqrt{t}},\frac{\kappa_t}{K}\}
\end{equation*}
with $\kappa_t=-\sum_{i=1}^n\langle\nabla_2L_i(g_{t-1}),\varphi_{\boldsymbol{\theta}_t}(\boldsymbol{x}_i)\rangle$, 
and let $\boldsymbol{\theta}_t$ be chosen according to Equation~\eqref{eq:paramit}. Then $\lim_{t\rightarrow\infty}\hat{\mathcal{R}}_n(g_t)=\hat{\mathcal{R}}_n(g_\star)$.
\end{theorem}

\begin{theorem}[Convergence of multi-output gradient boosting with line search]\label{th:gber2}
Suppose that Assumptions~\ref{assum:smooth} to \ref{assum:compact} are fulfilled and let $g_\star\in\argmin_{g\in\Span(\varphi_{\boldsymbol{\theta}})}\hat{\mathcal{R}}_n(g)$. Let the step size at iteration $t\geq 1$ be chosen as $b_t=\alpha \eta_t$, where $\alpha \in \interval[open left]{0}{1}$ is a shrinkage parameter, and $\eta_t$ is such that:
\begin{equation*}
\eta_t\in\argmin_{\eta\in\mathbb{R}^+} \hat{\mathcal{R}}_n(g_{t-1}+\eta\varphi_{\boldsymbol{\theta}_t})
\end{equation*}
with $\boldsymbol{\theta}_t$ chosen according to Equation~\eqref{eq:paramit}. Then $\lim_{t\rightarrow\infty}\hat{\mathcal{R}}_n(g_t)=\hat{\mathcal{R}}_n(g_\star)$.
\end{theorem}

More involved versions of these proofs would be needed to establish the Bayes consistency of multi-output gradient boosting. For an example of such an effort with regularization by early stopping in the univariate case, see \textcite[Chapter~10]{bach2024learning}.

\section{Parallel gradient boosting}\label{sec:pgb}

In this section, we describe the theoretical procedure of PGB. We then discuss the specific choices made in our implementation of the algorithm.

\subsection{General procedure}

As mentioned in the previous sections, direct implementations of gradient boosting for multi-output regression can be computationally demanding, because $M$ base models are typically used at each step of the descent. One way to deal with this issue is to divide each of these learning problems into several simpler ones. Namely, we propose to treat the direction of the predictions $\varphi_{\boldsymbol{\theta}}(\boldsymbol{x})$ separately from their norm. This is the core idea behind PGB. In PGB, we constrain the predictions of each of the base models $\varphi_{\boldsymbol{\theta}}$ to be collinear (hence the chosen name). The only learning task remaining is then to predict the corresponding norms for each data point, which is a univariate prediction task. The principle is illustrated in Figure~\ref{fig:gradient}.

\begin{figure}[!htbp]
\centering
\includegraphics[width=1\linewidth]{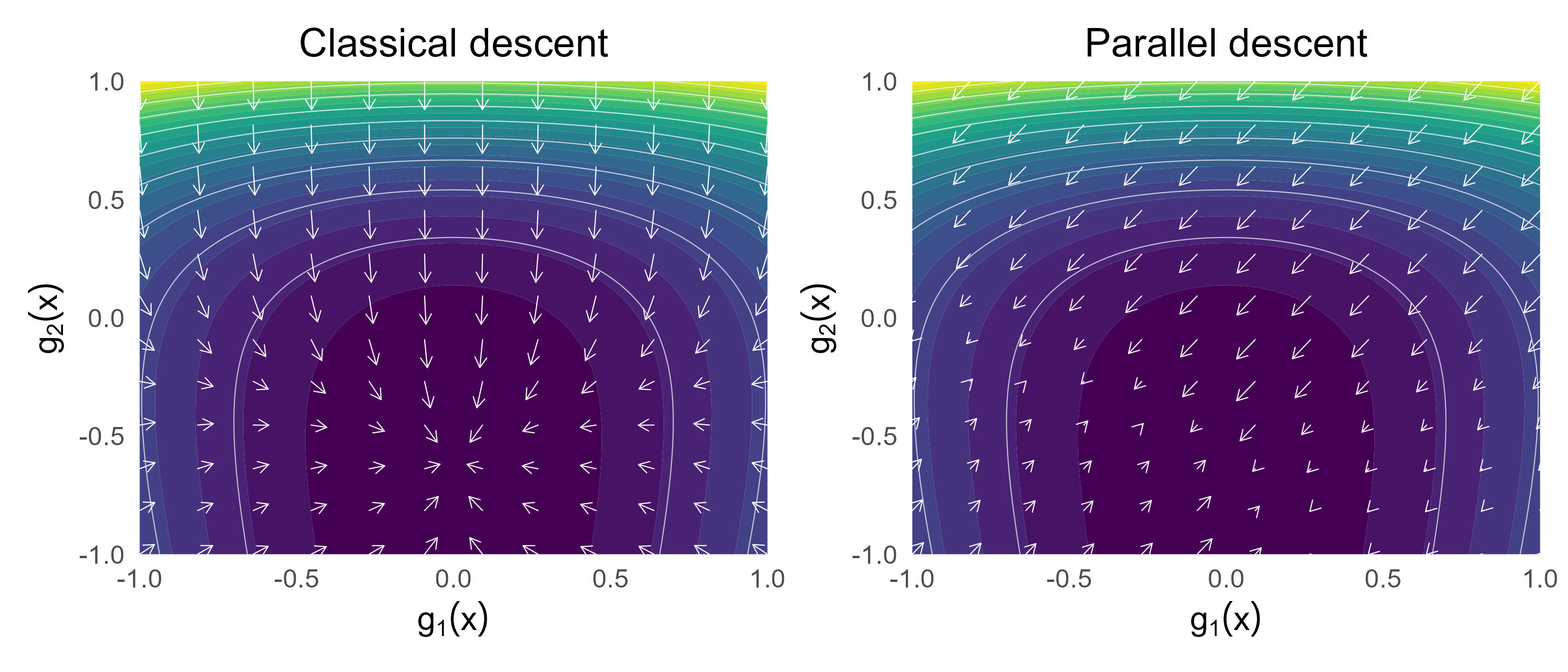}
\caption{\label{fig:gradient}Principle of parallel vs classical gradient descent. Here, $g_1(\boldsymbol{x})$ and $g_2(\boldsymbol{x})$ represent the first and second component of a prediction $g(\boldsymbol{x})$ on $\mathbb{R}^2$. The colors and contours represent the value of the loss function, while the white arrows are the targets to be learned: the negative gradient values in classical descent, and their projection on a given direction in parallel descent.}
\end{figure}

To formalize that, consider now a family of single-output models $(h_{\boldsymbol{w}})_{\boldsymbol{w}\in\mathcal{W}}$, where each $h_{\boldsymbol{w}}$ is a function from $\mathcal{X}$ to $\mathbb{R}$, and $\mathcal{W}$ is a finite-dimensional parameter space. For example, $(h_{\boldsymbol{w}})_{\boldsymbol{w}\in\mathcal{W}}$ can represent linear stumps or standard single-output decision trees. This single-output family will be used to build a multi-output one, in a way made precise below. Analogously to the multi-output weak learner assumption mentioned above, we make a \textit{single-output weak learner} assumption on $(h_{\boldsymbol{w}})_{\boldsymbol{w}\in\mathcal{W}}$. We suppose that the minimization problem
\begin{equation}\label{eq:paramit-uni}
   \boldsymbol{w}_\star\in\argmin_{\boldsymbol{w}\in\mathcal{W}} \sum_{i=1}^n a_i h_{\boldsymbol{w}}(\boldsymbol{x}_i)
\end{equation}
can be solved easily for any real $n$-tuple $(a_1,\ldots,a_n)$. This is the standard assumption in single-output gradient boosting when models such as decision trees are used as base learners.

In PGB, base models in $(h_{\boldsymbol{w}})_{\boldsymbol{w}\in\mathcal{W}}$ are used to set the norm of a final, multivariate prediction. To handle the direction of these predictions, we consider a distinct parameter $\boldsymbol{\beta}\in  \mathbb{S}^{M-1}$, where $\mathbb{S}^{M-1}$ is the unit $M-1$-sphere. The family of multi-output base models induced by $(h_{\boldsymbol{w}})_{\boldsymbol{w}\in\mathcal{W}}$, which is used in PGB, is defined as the family $(\varphi_{\boldsymbol{\theta}})_{\boldsymbol{\theta}\in\Theta}$ such that $\Theta=\mathcal{W}\times\mathbb{S}^{M-1}$ and for each $\boldsymbol{\theta}=(\boldsymbol{w},\boldsymbol{\beta})\in\Theta$ and $\boldsymbol{x}\in\mathcal{X}$,
\begin{equation}\label{eq:wbfamily}
\varphi_{\boldsymbol{\theta}}(\boldsymbol{x})=\varphi_{\boldsymbol{w},\boldsymbol{\beta}}(\boldsymbol{x})=h_{\boldsymbol{w}}(\boldsymbol{x})\boldsymbol{\beta}.
\end{equation}
One can easily check that (by construction), $\Vert\varphi_{\boldsymbol{\theta}}(\boldsymbol{x})\Vert_2=|h_{\boldsymbol{w}}(\boldsymbol{x})|$, and $\boldsymbol{\beta}$ gives the direction of the prediction. Such multi-output base models take values in $\mathbb{R}^M$, as in the previous section. The space of potential predictors explored in PGB lies in the linear span of this family, as in Equation~\eqref{eq:gbm-form}. Although only one single-output model is trained at each iteration, this space is, in many cases, as rich as in classical multi-output gradient boosting. A notable example is when $M$ independent decision trees are fitted at each iteration, one for each component of the predictions: the same space of predictors can be explored by PGB, despite using only one tree per iteration. This is expressed more formally in Proposition~\ref{prop:space}, whose proof is straightforward and omitted for brevity.
\begin{proposition}\label{prop:space}
The linear span of the set of base functions given in Equation~\eqref{eq:wbfamily} is $\Span(h_{\boldsymbol{w}})^M$, where $\Span(h_{\boldsymbol{w}})$ denotes the linear span of the set of base models in $(h_{\boldsymbol{w}})_{\boldsymbol{w}\in\mathcal{W}}$.
\end{proposition}

If one wanted to follow the procedure described in the previous section using the family thus defined, the multi-output weak learner hypothesis would be that the negative gradient can be maximized in a relatively easy way, that is:
\begin{align}\label{eq:minthetabeta}
(\boldsymbol{w}_\star,\boldsymbol{\beta}_\star)_t&\in\argmin_{(\boldsymbol{w},\boldsymbol{\beta})\in\mathcal{W}\times\mathbb{S}^{M-1}} \sum_{i=1}^n\langle\nabla_2L_i(g_{t-1}),\varphi_{\boldsymbol{w},\boldsymbol{\beta}}(\boldsymbol{x}_i)\rangle\\
    &=\argmin_{(\boldsymbol{w},\boldsymbol{\beta})\in\mathcal{W}\times\mathbb{S}^{M-1}} \sum_{i=1}^n h_{\boldsymbol{w}}(\boldsymbol{x}_i)\langle\nabla_2L_i(g_{t-1}),\boldsymbol{\beta}\rangle\nonumber
\end{align}
can be found efficiently. This requirement is strong in the general case, and especially when $(h_{\boldsymbol{w}})_{\boldsymbol{w}\in\mathcal{W}}$ corresponds to nonlinear models such as decision trees. Although specialized procedures could be designed to handle this problem, they would typically require repeated calls to the univariate learning algorithm (the \enquote{weak learner oracle}), hence undermining the algorithmic benefits of the whole approach. However, the minimization problem in $\boldsymbol{w}$ only (i.e. with $\boldsymbol{\beta}$ being fixed) is affordable under the single-output weak learner assumption. In PGB, only the latter is therefore made. The problem is then to find a way of approximately solving Equation~\eqref{eq:minthetabeta} under this hypothesis only.

One could imagine different types of approaches to address this problem, which would give rise to as many variants of PGB. In light of the preceding discussion, an appealing approach, that we follow in this work, is to choose a suboptimal value $\tilde{\boldsymbol{\beta}}_t$ for $\boldsymbol{\beta}_t$ and then solve Equation~\eqref{eq:minthetabeta} exactly under this choice. After that, the direction itself can be optimized in light of the value determined for $\boldsymbol{w}_t$, which may help the procedure converge without requiring an additional call to the weak learner oracle. As we discuss in Appendix~\ref{sec:strategy}, this idea has been indirectly explored in a few previous works, but without dedicated analysis. Crucially, it raises the question of how to select the value for $\tilde{\boldsymbol{\beta}}_t$, a question that, to our knowledge, has never been addressed explicitly either. In our implementation of PGB, we propose to use vectors of the canonical basis of $\mathbb{R}^M$ for that. Our motivation is twofold. First, such a choice is computationally lightweight, as it only requires maintaining one column of the matrix of gradient values (or pseudo-residuals) at each iteration. Second, this configuration enables convergence of the procedure in empirical risk, as we now discuss.

Our strategy for implementing PGB is based on the following requirements, which are used to ensure convergence. First, as said above, we require the use of any sequence of projection directions $(\tilde{\boldsymbol{\beta}}_t)_{t\in\mathbb{N}^*}$ such that for each $t\geq 1$, $\tilde{\boldsymbol{\beta}}_t$ is an element of the canonical basis of $\mathbb{R}^M$, i.e. 
\begin{equation}\label{eq:basis}
\tilde{\boldsymbol{\beta}}_t=\mathbf{e}_{m_t}
=\bigl(\underbrace{0,\ldots,0}_{\mathclap{m_t-1\ \text{times}}},\,1,\,\underbrace{0,\ldots,0}_{\mathclap{M-m_t\ \text{times}}}\bigr)
\end{equation}
for some $m_t\in\{1,\ldots,M\}$. Additionally, we require that each element of the canonical basis can be \enquote{visited} an infinite number of times. That is, we assume that there are $M$ strictly increasing functions $\rho_1,\ldots,\rho_M$, each from $\mathbb{N}^*$ to $\mathbb{N}^*$, such that for each $m\in\{1,\ldots,M\}$ and $s\geq1$,
\begin{equation}\label{eq:rho}
\tilde{\boldsymbol{\beta}}_{\rho_m(s)}=\mathbf{e}_m.
\end{equation}
At iteration $t\geq1$, the single-output model is then chosen as in Equation~\eqref{eq:minthetabeta}:
\begin{equation}\label{eq:wttilde1}
\boldsymbol{w}_t\in\argmin_{\boldsymbol{w}\in\mathcal{W}} \sum_{i=1}^n h_{\boldsymbol{w}}(\boldsymbol{x}_i)\langle\nabla_2L_i(g_{t-1}),\tilde{\boldsymbol{\beta}}_t\rangle.
\end{equation}
Finally, the multi-output prediction rule is updated as:
\begin{equation}\label{eq:wttilde2}
g_t=g_{t-1}+b_th_{\boldsymbol{w}_t}{\boldsymbol{\beta}}_t
\end{equation}
where (analogously to Theorem~\ref{th:gber2}), $b_t=\alpha\eta_t$, with $\alpha\in\interval[open left]{0}{1}$ being a fixed shrinkage parameter, and $\eta_t$ and $\boldsymbol{\beta}_t$ are given by:
\begin{equation*} (\eta_t,\boldsymbol{\beta}_t)\in\argmin_{(\eta,\boldsymbol{\beta})\in\mathbb{R}^+\times\mathbb{S}^{M-1}} \hat{\mathcal{R}}_n(g_{t-1}+\eta h_{\boldsymbol{w}_t}\boldsymbol{\beta}).
\end{equation*}
To simplify notations in what follows, we write $\boldsymbol{\gamma}_t=\eta_t\boldsymbol{\beta}_t$. Note that in this case, we have:
\begin{equation*}
\boldsymbol{\gamma}_t\in\argmin_{\boldsymbol{\gamma}\in\mathbb{R}^M} \hat{\mathcal{R}}_n(g_{t-1}+h_{\boldsymbol{w}_t}\boldsymbol{\gamma}).
\end{equation*}

One can see the determination of the value of $\eta_t$ and $\boldsymbol{\beta}_t$ as a form of multidimensional line search. The reliance of this implementation of PGB on such a line search makes it especially convenient when working with separable loss functions because, in this case, each component of $\boldsymbol{\gamma}_t$ can be determined separately. For analogous reasons, working with a separable loss function is also useful for studying the convergence of the algorithm. This motivates the following additional assumptions.
\begin{assumption}[Symmetry of univariate models]\label{assum:symuni}
    For each $\boldsymbol{w}\in\mathcal{W}$, there is a $\bar{\boldsymbol{w}}\in\mathcal{W}$ such that $h_{\boldsymbol{w}}=-h_{\bar{\boldsymbol{w}}}$.
\end{assumption}

\begin{assumption}[Separability]\label{assum:sep}
    The loss $L$ is separable, i.e., is such that:
\begin{equation*}
L(y,g_t(\boldsymbol{x})) = \sum_{m=1}^M\ell_m(y,g_{t,m}(\boldsymbol{x}))
\end{equation*}
    for some functions $\ell_m:\mathbb{R}^2\rightarrow \mathbb{R}^+$, where $g_{t,m}(\boldsymbol{x})$ denotes the $m$-th component of $g_t(\boldsymbol{x})$. Additionally, each of the $\ell_m$'s is convex and K-smooth in its second argument.
\end{assumption}

In this setting, we have the following convergence result, which is proved in Appendix~\ref{sec:proofs}. The idea is merely to apply Theorem~\ref{th:gber2} to each component of the loss function. In this case, the inner product $\langle \nabla_2L_i(g_{t-1}),\boldsymbol{\beta}_t \rangle$ gives the $i$-th pseudo-residual that would be used at iteration $t$ if one wanted to follow the univariate boosting procedure for the $m$-th component only.

\begin{theorem}[Convergence of parallel gradient boosting]\label{th:lspgb}
Suppose that Assumptions~\ref{assum:bound}, \ref{assum:symuni} and \ref{assum:sep} are fulfilled, as well as Assumption~\ref{assum:compact} applied to each elementary loss function $\ell_m$. Let $\alpha\in\interval[open left]{0}{1}$. Let $g_\star\in\argmin_{g\in\Span(\varphi_{\boldsymbol{\theta}})}\hat{\mathcal{R}}_n(g)$. 
At iteration $t\geq 1$, let $\tilde{\boldsymbol{\beta}}_t$ be chosen as described in Equation~\eqref{eq:rho}, and let $\boldsymbol{w}_t$ be chosen according to Equation~\eqref{eq:paramit-uni}, i.e.,
\begin{equation*}
\boldsymbol{w}_t\in\argmin_{\boldsymbol{w}\in\mathcal{W}} \sum_{i=1}^n h_{\boldsymbol{w}}(\boldsymbol{x}_i)\langle\nabla_2L_i(g_{t-1}),\tilde{\boldsymbol{\beta}}_t\rangle.
\end{equation*}
Finally, let the multidimensional step size $\boldsymbol{\gamma}_t$ be chosen as:
\begin{equation*}
\boldsymbol{\gamma}_t\in\argmin_{\boldsymbol{\gamma}\in\mathbb{R}^M} \hat{\mathcal{R}}_n(g_{t-1}+h_{\boldsymbol{w}_{t}}\boldsymbol{\gamma})
\end{equation*}
and the model be updated as:
\begin{equation*}
g_t = g_{t-1} + \alpha h_{\boldsymbol{w}_{t}}\boldsymbol{\gamma}_t.
\end{equation*}
Then $\lim_{t\rightarrow\infty}\hat{\mathcal{R}}_n(g_t)=\hat{\mathcal{R}}_n(g_\star)$.
\end{theorem}

In Appendix~\ref{sec:linear}, we illustrate the convergence of our implementation of PGB when boosting linear models, thereby showing that our general formulation of the procedure applies beyond the most classical class of base learners considered in the literature, namely decision trees.

\subsection{Implementation details}

The practical steps of the PGB procedure that we propose are given in Algorithm~\ref{alg:pgb}. In practice, instead of choosing the descent directions $\boldsymbol{\beta}_t$ in a deterministic way, we draw them at random. The sampling distribution for the directions is uniform, so that the condition of \enquote{infinite visits} of the canonical vectors in $\mathbb{R}^M$ is almost surely met. In Appendix~\ref{sec:strategy}, we explore the practical consequences of this strategy in terms of training times and test error for multiple quantile regression. We find that this choice makes PGB models especially fast to train, with a predictive performance comparable to or better than that obtained through alternative strategies.

In our implementation of the algorithm, we use decision trees as base learners, with leaf-wise line search \parencite{friedman2001greedy}. These are fitted using a pre-sorted, histogram-based algorithm to improve training speed \parencite{ranka1998clouds, ke2017lightgbm}. As with classical gradient boosted decision trees \parencite{si2017gradient}, they aim to solve the minimization problem stated in Equation~\eqref{eq:paramit-uni} through an approximate objective, namely the minimization of the residual sum of squares. More precisely, the actual problem that each tree tries to solve is
\begin{equation*}
\boldsymbol{w}_\star\in\argmin_{\boldsymbol{w}\in\mathcal{W}} \sum_{i=1}^n (-a_i-h_{\boldsymbol{w}}(\boldsymbol{x}_i))^2
\end{equation*}
where $(a_1,\ldots,a_n)$ are projected pseudo-residuals. This idea is treated in more detail in \textcite[Section~7.4.3]{schapire2013boosting}, and its consequences on the convergence of the univariate gradient boosting algorithm are explored in \textcite{biau2021optimization}.

We also allow subsampling at each iteration, as in stochastic gradient boosting \parencite{friedman2002stochastic}. In this variant, a given fraction of training observations are sampled (without replacement) at each iteration, and the base model for this iteration is fitted on this sample. Additionally, we can further limit the number of training observations used for the line search (for example, to a few hundred). We observed that this strategy reduces training times without degrading the model performance. As for the early stopping parameter $T$, it can be determined using a validation set, or by cross-validation.

Finally, we note that it would also be possible to stop updating individual components of the loss when the validation loss for these components stops decreasing. This would result in a potentially different number of boosting iterations for each component. We do not investigate this option here, because in our use case for PGB (namely, estimation of conditional distributions) using a common early stopping parameter provides some regularization in the $Y$ direction and is therefore useful.

\begin{algorithm}
\DontPrintSemicolon
 \KwData{a training dataset $\mathscr{D}_n = \big((\boldsymbol{x}_1,y_1),\ldots,(\boldsymbol{x}_n,y_n)\big)$, a family of univariate regression functions $(h_{\boldsymbol{w}})_{\boldsymbol{w}\in\mathcal{W}}$, a differentiable and separable loss function $L:\mathbb{R}\times\mathbb{R}^M\rightarrow\mathbb{R}^+$, a maximum number of iterations $T$, a shrinkage parameter $\alpha\in\interval[open left]{0}{1}$. }
 \KwResult{ a prediction function $g_T:\mathcal{X}\rightarrow \mathbb{R}^M$. }
 Initialize $g_0(\boldsymbol{x})=\argmin_{\boldsymbol{\gamma}\in\mathbb{R}^M}\sum_{i=1}^nL(y_i,\boldsymbol{\gamma}).$\;
 \For{$t=1,\ldots,T$}{
    \For{$i = 1, \ldots, n $}{
        Compute the pseudo-residuals for observation $i$:
        \begin{algomathdisplay}
        \nabla_2L_i(g_{t-1})=\bigg[\frac{\partial L(y_i,g(\boldsymbol{x}_i))}{\partial g(\boldsymbol{x}_i)}\bigg]_{g=g_{t-1}}.
        \end{algomathdisplay}
    }
    Sample $\tilde{\boldsymbol{\beta}}_t$ uniformly from $\{\mathbf{e}_1,\ldots,\mathbf{e}_M\}.$ \;
    Choose $\boldsymbol{w}_t$ as:
    \begin{algomathdisplay}
    \boldsymbol{w}_t\in\argmin_{\boldsymbol{w}\in\mathcal{W}}\sum_{i=1}^n h_{\boldsymbol{w}}(\boldsymbol{x}_i)\langle\nabla_2L_i(g_{t-1}), \tilde{\boldsymbol{\beta}}_t\rangle.
    \end{algomathdisplay}
    Perform a multidimensional line search:
    \begin{algomathdisplay}
    \boldsymbol{\gamma}_{t}\in\argmin_{\boldsymbol{\gamma}\in\mathbb{R}^M} \sum_{i=1}^nL(y_i,g_{t-1}(\boldsymbol{x}_i)+h_{\boldsymbol{w}_{t}}(\boldsymbol{x}_i)\boldsymbol{\gamma}).
    \end{algomathdisplay}
    Update $g_t=g_{t-1}+\alpha h_{\boldsymbol{w}_{t}}\boldsymbol{\gamma}_t$.\;
 }
 \Return{$g_T$.}\;
 \caption{Parallel gradient boosting.}
 \label{alg:pgb}
\end{algorithm}

\section{Description of the conditional distribution estimator}\label{sec:cde}

Our estimator of conditional distributions is based on multiple quantile regression. We first describe how PGB can be used to do so. Then, we discuss how to obtain estimates of conditional distributions from estimates of conditional quantiles. For both use cases, the PGB procedure is applied to binary trees as base learners.

\subsection{Parallel gradient boosting for multiple quantile regression}\label{sec:mqr}

One natural application of PGB is multiple quantile regression. In multiple quantile regression, a grid of quantile levels (or knots) $0<\tau_1<\ldots<\tau_M< 1$ is fixed, and one tries to estimate the conditional quantile function of $Y$ given $X$ on this grid. The classical univariate loss for this aim is the pinball loss \parencite{koenker2001quantile}, defined for a quantile $\tau_m$ and a prediction function $g:\mathcal{X}\rightarrow\mathbb{R}$ by:
\begin{align*}
\ell_{\tau_m}(y, g(\boldsymbol{x}))&=(\tau_m-\mathds{1}\{(y-g(\boldsymbol{x}))<0\})(y-g(\boldsymbol{x})) \\
&=\begin{cases}
(\tau_m-1)(y-g(\boldsymbol{x})) & \text{if } y < g(\boldsymbol{x}),\\
\tau_m(y-g(\boldsymbol{x}))  & \text{if } y\geq g(\boldsymbol{x}).
\end{cases}
\end{align*}
It can be shown \parencite{Steinwart2011-si} that under mild conditions on the conditional distribution of $Y$ given $X=\boldsymbol{x}$, the Bayes predictor for the loss $\ell_{\tau_m}$ is the conditional quantile $F_{Y\mid \boldsymbol{x}}^{-1}(\tau_m)=\inf\{y\in\mathbb{R},F_{Y\mid \boldsymbol{x}}(y)\geq \tau_m\}$\footnote{In practice, it suffices that the conditional quantile function $F_{Y\mid \boldsymbol{x}}^{-1}$ be strictly increasing, and that $Y$ have finite conditional moments of all orders given $X=\boldsymbol{x}$.}. The pinball loss is piecewise-linear in its second argument, so it does not meet the usual assumptions that are used in proofs of convergence of gradient boosting (besides being non-differentiable at 0, it notably lacks the $K$-smoothness property). Still, boosted models trained to minimize this loss have shown satisfactory performance in practice \parencite{bauer2024pinball, dudek2025multivariate, karakacs2025delivery, nzarigema2022quantile}. Contrary to alternative loss functions designed specifically for boosting (see, e.g., \cite{sluijterman2025composite}), the pinball loss is also easier to optimize for line search \parencite{fouillen2023proximal}, which is why we use it with PGB as well.

In a multi-output setting, where $g$ takes values in $\mathbb{R}^M$, one aims to minimize the sum of the pinball losses on the grid 
$\tau_1,\ldots,\tau_M$. If the grid is equally spaced (i.e. $\tau_1=1-\tau_{M}=\tau_{m+1}-\tau_m$ for each $m\in\{1,\ldots,M-1\}$), this sum is known (up to a factor) as the weighted interval score (WIS; \cite{bracher2021evaluating}):
\begin{equation*}
\operatorname{WIS}(y,g(\boldsymbol{x}))=\frac{1}{M}\sum_{m=1}^M \ell_{\tau_m}(y, g_m(\boldsymbol{x}))
\end{equation*}
where $g_m(\boldsymbol{x})$ denotes the $m$-th component of $g(\boldsymbol{x})$. By interpolating between quantiles, a prediction $g(\boldsymbol{x})$ of the conditional quantiles of $Y$ given $X=\boldsymbol{x}$ can be used to obtain an estimate of the conditional quantile function $F_{Y\mid \boldsymbol{x}}^{-1}$, and (by inversion), an estimate $\hat{F}_{Y\mid \boldsymbol{x}}$ of the conditional cumulative distribution function (c.d.f.) of $Y$ given $X=\boldsymbol{x}$. It can then be shown \parencite{Brehmer2021-ie, bracher2021evaluating, duchemin2025efficient} that as $M\rightarrow \infty$, minimizing the WIS with respect to $g(\boldsymbol{x})$ becomes equivalent to minimizing the continuous ranked probability score (CRPS), defined as:
\begin{equation*}
\operatorname{CRPS}(y, \hat{F}_{Y\mid \boldsymbol{x}}) = \int_{-\infty}^{+\infty}(\hat{F}_{Y\mid \boldsymbol{x}}(s)-\mathds{1}\{y\leq s\})^2ds.
\end{equation*}

The CRPS is a proper scoring rule, meaning that it is minimized by the actual conditional c.d.f. $F_{Y\mid \boldsymbol{x}}$ \parencite{waghmare2025proper}. This suggests using multiple quantile regression (on the densest possible quantile grid) to estimate conditional distributions \parencite{koenker2013distributional}, a task for which it would be tempting to use gradient boosting \parencite{sluijterman2025composite}. However, the computational cost of fitting one base learner per quantile and iteration makes implementations of the standard algorithm of little practical use. This is a use case for which PGB could be particularly well suited.

To illustrate the performance of PGB for multiple quantile regression, we fit a PGB model on a simulated dataset according to \textcite{lei2014distribution}, with a target of 20 quantiles regularly spaced on $\interval[open]{0}{1}$. The dataset is the same as in Figure~\ref{fig:dens}. We use a learning rate of 0.02, a subsampling rate of 0.5, and trees of depth 1 (stumps) with at least 5 observations per node. We discretize $X$ in 256 discrete bins, and use 256 random training observations in each leaf and at each iteration to perform the line search. We randomly select 20\% of the observations to serve as validation data for early stopping. We also fit an XGBoost model to the same data and with the same hyperparameters\footnote{Note that the second-order approximation of the loss on which XGBoost usually relies to speed up convergence cannot be used here because the pinball loss has second derivative of zero.}. The L1 and L2 regularization options of XGBoost are not used. The results of the experiment are represented in Figure~\ref{fig:crossing}.

\begin{figure}[!htbp]
\centering
\includegraphics[width=1\linewidth]{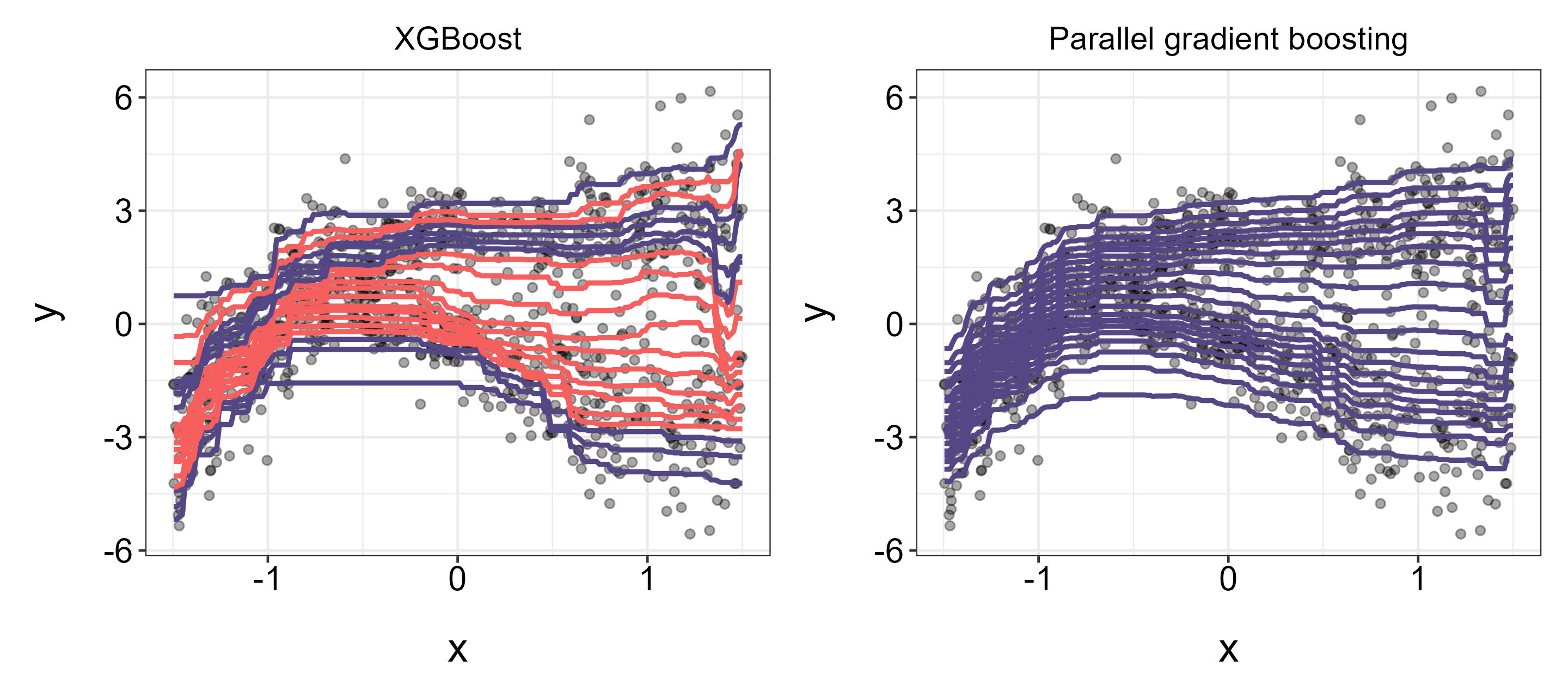}
\caption{\label{fig:crossing}Results of conditional quantile estimation with XGBoost and parallel gradient boosting on a simulated dataset. Quantile curves for which quantile crossing occurs are plotted in red.}
\end{figure}

In multiple quantile regression, quantile crossing is a well-known phenomenon by which $\hat{F}_{Y\mid \boldsymbol{x}}^{-1}(\tau_{m_1})<\hat{F}_{Y\mid \boldsymbol{x}}^{-1}(\tau_{m_2})$ for some $\boldsymbol{x}\in\mathcal{X}$, even though $\tau_{m_1} > \tau_{m_2}$. This needs to be corrected if one wants to estimate the entire conditional quantile function $F_{Y\mid \boldsymbol{x}}^{-1}$, because by definition the latter is increasing \parencite{chernozhukov2010quantile, moon2026monotone}. In Figure~\ref{fig:crossing}, we see that PGB results in no quantile crossing, whereas with XGBoost more than half of the quantiles are affected. We hypothesize that this is due to the fact that predictions for different quantiles share the same splits with PGB, as only one tree is used at each iteration for all of them. A similar observation has been made with at least one other method of boosting-based composite quantile regression \parencite{sluijterman2025composite}. Note that quantile crossing, while rare, can \textit{still} occur with PGB, especially with dense quantile grids.

To investigate the performance of PGB compared to XGBoost on real data, we run a second experiment. In this experiment, we train a PGB model and an XGBoost model to perform multiple quantile regression on the baseball dataset \parencite{he1998bivariate}, which has already been used by \textcite{liu2009stepwise} with a similar aim. More precisely, the task is to predict quantile values of the salaries of baseball players, given all other variables in the dataset. We run the experiment with an increasing number of regularly spaced quantiles. The hyperparameters used are the same as in the previous experiment, except that we boost trees of depth 3 instead of stumps. We randomly split the dataset into two parts of equal size, one of which is used to train the models and the other to estimate the generalization error. Among training observations, 20\% are used as a validation set for early stopping. The procedure is repeated 50 times for each number of quantiles. The results of the experiment are given in Figure~\ref{fig:baseball} in terms of computation times and WIS.

\begin{figure}[!htbp]
\centering
\includegraphics[width=1\linewidth]{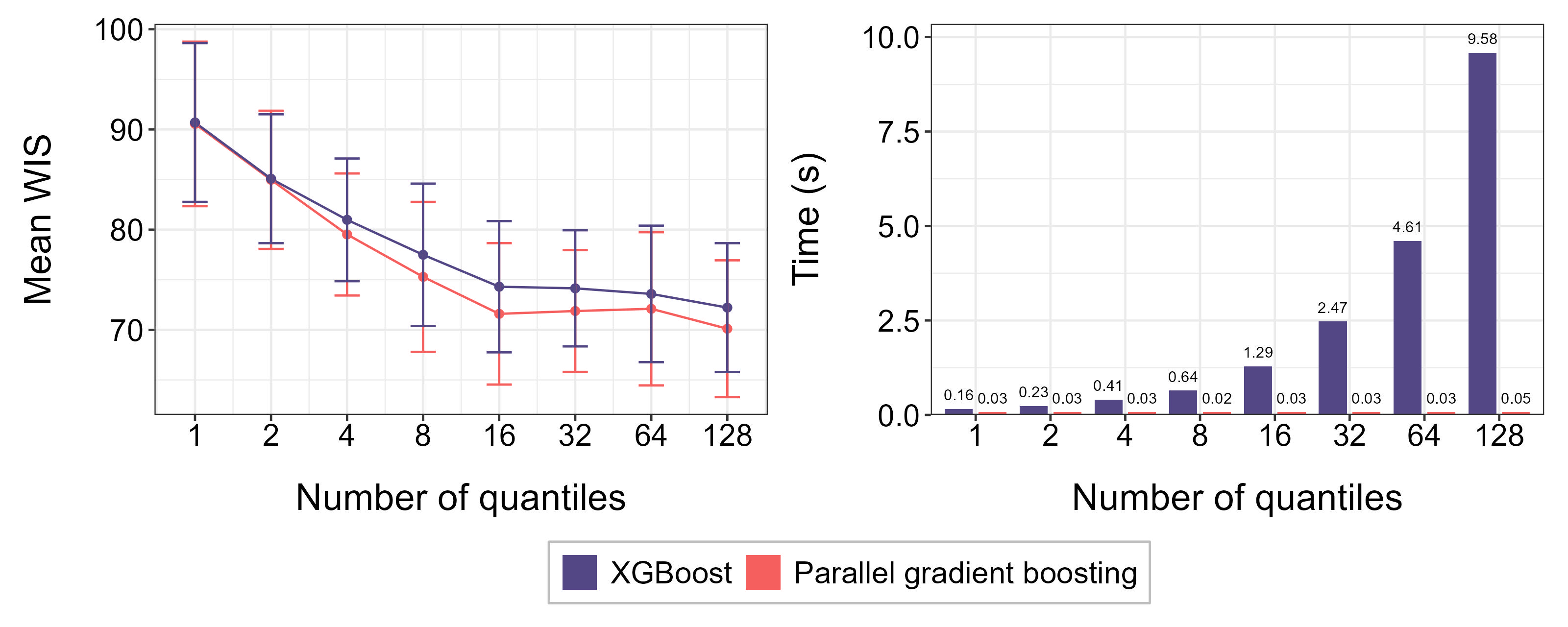}
\caption{\label{fig:baseball}Results of conditional quantile estimation with XGBoost and parallel gradient boosting on the baseball dataset, for different numbers of target quantiles. Left: mean weighted interval score (WIS). Right: computation times in seconds. The error bars represent one standard deviation.}
\end{figure}

Overall, XGBoost and PGB achieve similar test errors for all grids of quantiles. On the computational level, however, the performance gains of PGB are massive. Computation times vary linearly in the number of quantiles for both methods, but for PGB fitting the base models remains the main contributor to the overall computation times even for large numbers of quantiles. Hence, running the procedure for several tens of quantiles does not take significantly longer than for one quantile. In contrast, training an XGBoost model for 128 quantiles takes more than 50 times longer than training for one. This is especially impactful when one needs to run the procedure several times, for example for hyperparameter tuning.

\subsection{From conditional quantiles to conditional distributions}\label{sec:cqtocd}

Multiple quantile regression provides estimates of the value of the conditional quantile function on a given quantile grid. To derive estimates of the entire quantile function from these, one possibility, which we follow, is to interpolate between the evaluation points. In this case, one must first resolve any quantile crossings to obtain valid estimates of the function. 

Our post-processing approach to deal with quantile crossings is as follows. Once the PGB model is trained, suppose that for some $\boldsymbol{x}\in\mathcal{X}$, and some $m_1$ and $m_2$ in $\{1,\ldots,M\}$ such that $\tau_{m_1} > \tau_{m_2}$ we have $g_{m_1}(\boldsymbol{x})<g_{m_2}(\boldsymbol{x})$, where $g_{m_1}(\boldsymbol{x})$ and $g_{m_2}(\boldsymbol{x})$ represent the $m_1$ and $m_2$ components of $g(\boldsymbol{x})$. To predict the conditional distribution, we work with the isotonic regression operator $\mathbf{IR}$, defined by:
\begin{equation*}
\mathbf{IR}g(\boldsymbol{x}) = \argmin_{\substack{\boldsymbol{q}\in\mathbb{R}^M \\ q_1 \leq \ldots \leq q_M }} \Vert g(\boldsymbol{x})-\boldsymbol{q} \Vert_2^2.
\end{equation*}
That is, we use the optimal projection of the original prediction on the set of ordered quantile values in the sense of the Euclidean norm. There already exist efficient algorithms to solve such problems \parencite{pardalos1999algorithms}. Isotonic regression has already been applied to multiple quantile regression \parencite{jokiel2024estimation}, and in this setting, \textcite{fakoor2023flexible} proved that it can only improve the WIS. 

The most natural way of interpolating between appropriately processed quantiles is perhaps by linear interpolation. By inverting the estimated conditional quantile function, one obtains a piecewise-linear estimate of the conditional distribution function. Note that when $\tau_1 \neq 0$ and/or $\tau_M \neq 1$, as is the case when minimizing the pinball loss, the interpolated c.d.f. is discontinuous, hence it does not naturally give rise to a Lebesgue density function. However, one can map every such c.d.f. to a valid Lebesgue p.d.f., with an arbitrarily small bias as $M$ increases. Starting from an interpolated c.d.f. $\hat{F}_{\operatorname{raw}}$ with knots $\tau_1,\ldots,\tau_M$ (we omit the subscript $Y \mid \boldsymbol{x}$ for clarity, although this discussion applies to conditional c.d.f.'s), we define its scaled version as:
\begin{equation*}
   \hat{F}_{\operatorname{scaled}}(y) =\begin{cases}
\hat{F}_{\operatorname{raw}}(y)  & \text{if }\hat{F}_{\operatorname{raw}}(y)\in\{0,1\},\\
\frac{M+1}{M-1}(\hat{F}_{\operatorname{raw}}(y)-\frac{1}{M+1}) & \text{otherwise}.
\end{cases}
\end{equation*}
It is easy to verify that $\hat{F}_{\operatorname{scaled}}$ is absolutely continuous if the processed quantile values are strictly increasing. The resulting p.d.f. is used as a conditional p.d.f. estimate associated with the multiple quantile regression performed. Unfortunately, this conditional p.d.f. is piecewise-constant, which is generally undesirable, as it is often assumed that the true function is smooth. Additionally, it is supported only on a compact subset of $\mathbb{R}$, hence not modeling the tails of the distribution and increasing bias. Figure~\ref{fig:interpol} illustrates the relationship between estimated conditional quantiles, conditional c.d.f.'s and conditional p.d.f.'s obtained with this strategy.

\begin{figure}[!htbp]
\centering
\includegraphics[width=1\linewidth]{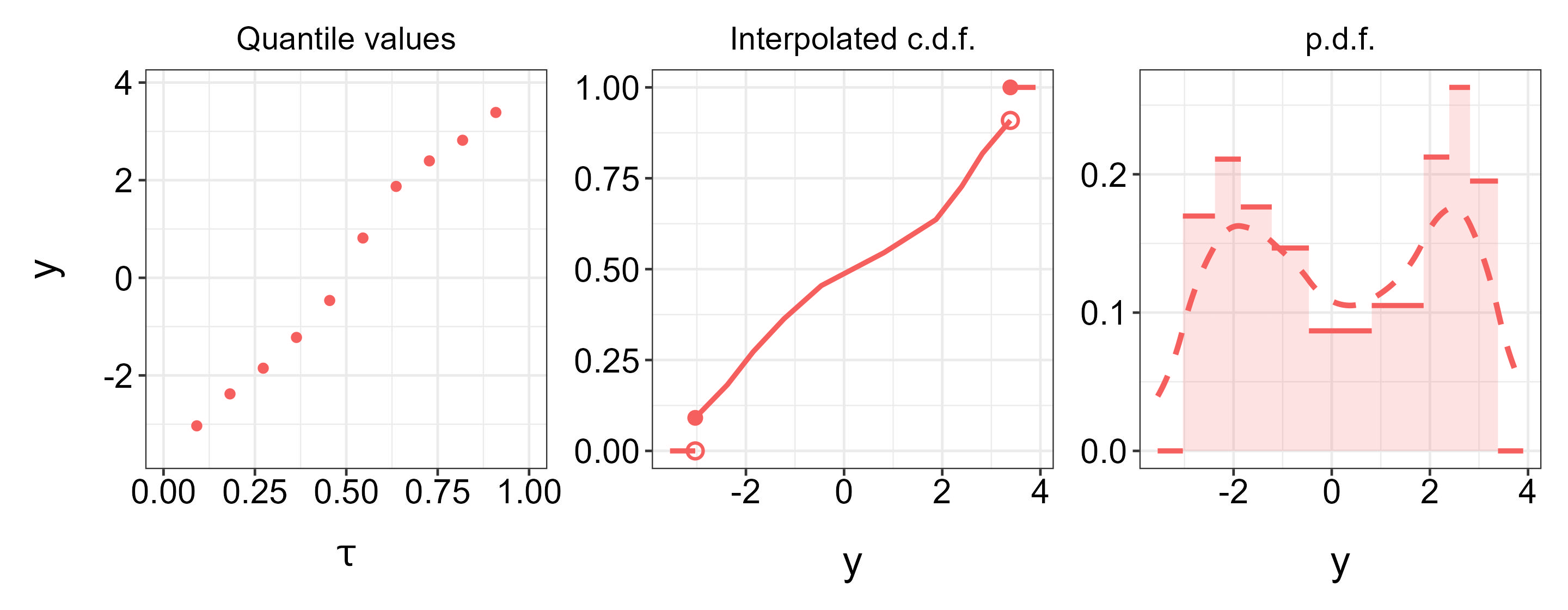}
\caption{\label{fig:interpol}Illustration of the linear interpolation between quantiles, and the resulting estimate of the conditional p.d.f. The smooth, dashed estimate on the right panel is that obtained by smoothing the piecewise-constant one using the method described in Section~\ref{sec:cqtocd}. The estimation is that of $F_{Y\mid 0.77}$, with the data of Figure~\ref{fig:dens} and a grid of 10 quantiles.}
\end{figure}

To obtain smooth estimates of conditional p.d.f.'s supported on the whole real line, our approach is to convolve the initial piecewise-constant estimate of the conditional p.d.f. with a smoothing kernel. In practice, we use a Laplacian kernel, defined by:
\begin{equation*}
K_\lambda(y,y')=\frac{\lambda}{2}\exp(-\lambda\vert y-y'\vert)
\end{equation*}
for each $(y,y')\in\mathbb{R}^2$, where $\lambda$ is a smoothing parameter known as the bandwidth of the kernel. The use of this kernel rather than more traditional ones like the Gaussian kernel is to simplify computations of the convolution. We use the same bandwidth for every value of the covariates $X$. This does not preserve the quantile values on the grid that has been used to perform multiple quantile regression, but instead provides some regularization in the $Y$ direction. For this reason, we consider the bandwidth of the kernel as a crucial hyperparameter, and choose its value by cross-validation to minimize the logarithmic score \parencite{waghmare2025proper}. For an estimate $\hat{f}_{Y\mid \boldsymbol{x}}$ of the conditional p.d.f. of $Y$ given $X=\boldsymbol{x}$ (seen as a function of $\boldsymbol{x}$), it is defined as:
\begin{equation*}
\operatorname{LogS}(y,\hat{f}_{Y\mid \boldsymbol{x}})=-\log(\hat{f}_{Y\mid \boldsymbol{x}}(y)).
\end{equation*}

The logarithmic score is strongly tied to the likelihood function for parametric models. As such, it is computationally affordable to determine its optimal value on out-of-fold data, even for large values of $M$. In practice, our estimator uses a regular grid of $M=\max(20, 2\times\sqrt{n})$ (rounded to the nearest integer) quantiles. Our hypothesis is that this rule-of-thumb is enough to ensure that the approximation error (due to interpolating between quantile values) is dominated by the estimation error of the learning algorithm. As the latter has not been studied with PGB, we do not attempt to prove the theoretical soundness of choosing such number of quantiles. We instead provide experimental results supporting this choice in  Appendix~\ref{sec:ntargets}.

\section{Experiments}\label{sec:experiments}

In this section, we report several experiments on synthetic as well as real datasets in order to assess the relative performance of PGB to estimate conditional distributions. PGB is compared to 3 other nonparametric or semiparametric estimators, namely:
\begin{itemize}
    \item Distributional random forests (DRF; \cite{cevid2022distributional}). This is a forest-based estimator inspired by quantile regression forests \parencite{meinshausen2006quantile}, but with a maximum mean discrepancy splitting criterion which is also suitable for multivariate distributions. The native algorithm provides stepwise-constant estimates of the conditional c.d.f.'s, where the steps correspond to the values of the dependent variable observed in the training sample. In the experiments, we use a smoothed version of these estimates using a Gaussian kernel, with a bandwidth determined by maximizing the out-of-bag logarithmic score. We use 5000 trees and leave the other hyperparameters at their default value.

    \item FlexCode \parencite{Izbicki2017-lo}. This estimator involves estimating the coefficients of an orthogonal series expansion of the unknown density function, thereby framing the conditional density estimation problem as a multiple regression one. Here we use the Lasso \parencite{tibshirani1996regression} version of the estimator with up to 25 Fourier basis functions, and other hyperparameters determined on a validation set of 20\% randomly selected observations.

    \item LinCDE \parencite{gao2022lincde}. This method is based on multiple Poisson regressions, which are performed using specialized binary trees. The regression results are combined sequentially in a boosting-like fashion. The algorithm admits many hyperparameters. In our experiments, we use trees of depth $\min(d,3)$, with $d$ being the dimension of the feature set, and leave the other hyperparameters at their default value (including the early stopping parameter set to 0.2).
\end{itemize}
The hyperparameters used for PGB are the same as in the experiments of Section~\ref{sec:mqr}, except for the depth of trees which is chosen as $\min(d,3)$. The number of boosting iterations is determined by 5-fold cross-validation. Figure~\ref{fig:denscomp} gives an example of a predicted conditional density by each of these methods based on the data of Figure~\ref{fig:dens}.

\begin{figure}[!htbp]
\centering
\includegraphics[width=0.7\linewidth]{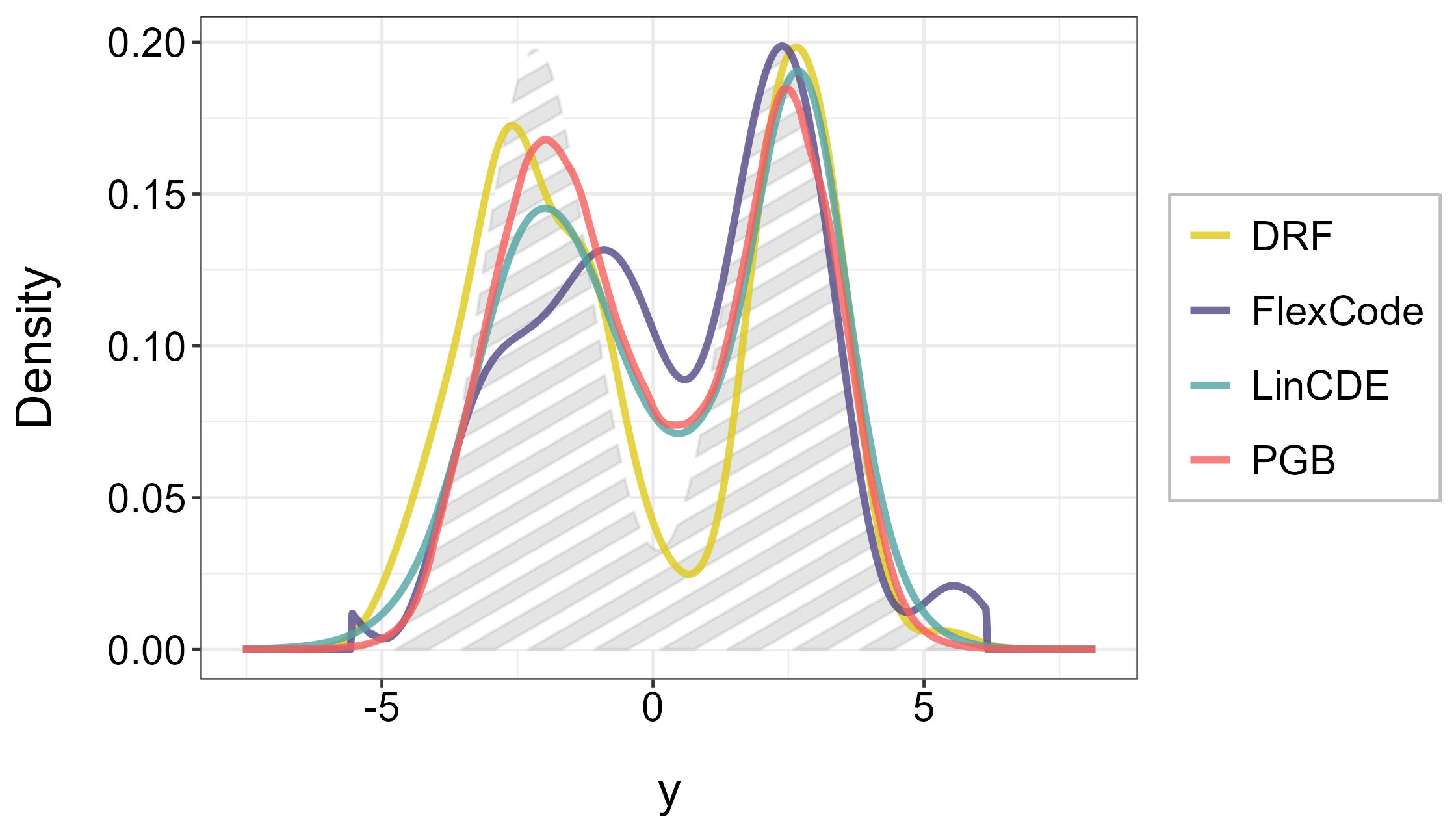}
\caption{\label{fig:denscomp}Prediction of the conditional density of $Y$ given $X=0.77$ by each of the candidate methods, with the data of Figure~\ref{fig:dens}. The dashed area corresponds to the real conditional density (the oracle).}
\end{figure}

The code for all experiments presented in this paper was written in R, and is available at \url{https://github.com/rchapelle/PGB}. It was executed on a standard personal computer, without parallelization to make the computation times easier to interpret. The corresponding software package was used for each candidate estimator. In these implementations, DRF, FlexCode and LinCDE do not support missing values or categorical predictors. The former are therefore imputed by the mean (for numeric variables) or mode (for categorical ones), while the latter are binarized into indicator variables.

\subsection{Performance on simulated data}\label{sec:expsim}

In this experiment, we draw 100 realizations from known distributions with an increasing number of covariates. This modest number of observations is used because it makes the estimation of conditional distributions more challenging, as it worsens the signal-to-noise ratio. An independent test set of 100 observations from the same distribution is used to estimate the generalization error. Each experiment is repeated 50 times. We report the mean test error and computation times for each method. The latter comprise the time needed for training (including hyperparameter tuning) and that of prediction on the test set. The loss function used to evaluate the predictions of the candidate methods is the integrated squared error (ISE), given by:
\begin{equation}\label{eq:ise}
\operatorname{ISE}(\hat{f}_{Y\mid\cdot}) = \iint(\hat{f}_{Y\mid \boldsymbol{x}}(y)-f_{Y\mid \boldsymbol{x}}(y))^2dydF_X(\boldsymbol{x})
\end{equation}
where $f_{Y\mid \boldsymbol{x}}$ is the conditional p.d.f. of $Y$ given $X=\boldsymbol{x}$, and $\hat{f}_{Y\mid \boldsymbol{x}}$ is the corresponding estimate. For a fixed value $\boldsymbol{x}$ of $X$, we approximate the inner integral $ \int(\hat{f}_{Y\mid \boldsymbol{x}}(y)-f_{Y\mid \boldsymbol{x}}(y))^2dy $ on a grid of 500 points by trapezoidal integration. In this experiment, we consider the following data generating processes, which are inspired by \textcite{Izbicki2017-lo}:
\begin{itemize}
    \item \textbf{Irrelevant covariates.} Here, we let
    \begin{equation*}
    X=(X_1,\ldots,X_d)\sim\mathcal{N}((1,\ldots,1),0.25\mathbf{I}_d)
    \end{equation*}
    and $Y\mid \boldsymbol{x} \sim \mathcal{N}(x_1,(0.25\lvert x_2\rvert+0.1)^2)$. Hence only the first two components of $X$ influence $Y$, respectively through its conditional mean and variance.
    \item \textbf{Data on manifold.} We introduce two parameters $B_1$ and $B_2$ each lying on the ${d-1}$-sphere, whose values are determined randomly before each replication of the experiment as follows. First, $B_1$ is sampled uniformly on the ${d-1}$-sphere. Then, $B_2$ is sampled uniformly among unit $d$-vectors orthogonal to $B_1$. We also introduce a latent random variable $Z\sim\mathcal{U}(\interval{0}{2\pi})$. The covariates $X$ are given by $X=\cos(Z)B_1+\sin(Z)B_2$, and we let $Y\mid z\sim\mathcal{N}(z,(\lvert \cos(z)+\sin(z) \rvert +0.1)^2)$. Hence $Y$ depends on $X$ only through a one-dimensional latent variable.
    \item \textbf{Non-sparse data.} Here, we let $X\sim\mathcal{N}((0,\ldots,0),\mathbf{I}_d)$ and let
    \begin{equation*}
    Y\mid \boldsymbol{x} \sim \mathcal{N}(\frac{1}{d}\sum_{j=1}^dx_j,(0.1+\frac{1}{d}\sum_{j=1}^d\lvert x_j \rvert)^2).
    \end{equation*}
    In this setting, all covariates are equally important to predict $Y$.
\end{itemize}

The results of this experiment are given in Figure~\ref{fig:simres}. As expected, the performance of all methods tends to decrease in the irrelevant covariates setting when the dimensionality of the features increases, whereas it improves in the non-sparse setting. With data on manifold, all methods maintain their performance when $d$ increases. This is unsurprising, because all of the underlying learning methods (namely, ensemble learning and the Lasso) are designed to adapt to the intrinsic dimensionality of the data \parencite[Chapter~16]{hastie2009elements}. Regarding the relative performances of the methods in these three settings, we see that PGB and LinCDE perform significantly better than DRF and FlexCode in the presence of irrelevant covariates and with data on manifold. This relates to the properties of boosting, as opposed to e.g. random forests in such settings, which have been thoroughly described in the literature \parencites[Chapter~15]{hastie2009elements}{elith2008working, revelas2025random}. In contrast, PGB and DRF show better performance than the two other methods in the non-sparse setting. Hence PGB is competitive (and even best-performing) in each scenario, which illustrates its flexibility. As for computation times, they are significantly lower in this experiment with PGB compared to the other three methods. With PGB, there are clear differences in the relationship between computation times and the dimensionality of the covariates depending on the type of data generating process. This is because depending on the setting, the algorithm may use more trees as $d$ increases, which reflects its adaptivity to the complexity of the data at hand.

\begin{figure}[!htbp]
\centering
\includegraphics[width=1\linewidth]{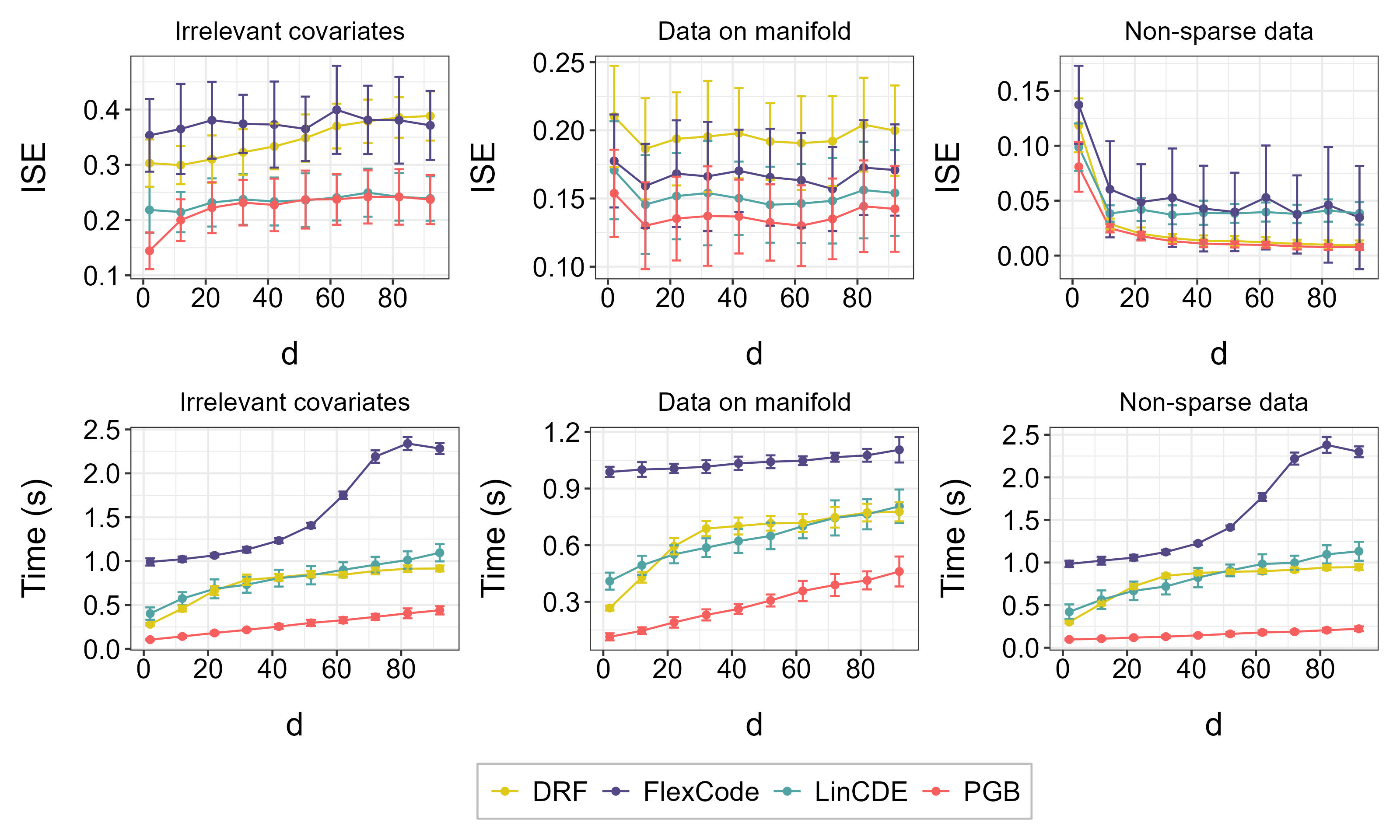}
\caption{\label{fig:simres}Results of conditional density estimation on simulated datasets with an increasing number of covariates. The error bars represent one standard deviation.}
\end{figure}

\subsection{Importance scores}

When predicting probability distributions, it is often of interest to know which covariates have the most influence on the predictions. This can help refine the model, improve its interpretability, and provide valuable insights in data mining settings. For that, one usually attributes an importance score to each variable on which the model is fitted. For tree-based methods, various heuristics can be used to compute such scores, based e.g. on the gains in the splitting criterion \parencites[Chapter~5]{breiman2017classification}{strobl2007bias}, random permutations \parencite{kaneko2022cross}, or Shapley values \parencite{lundberg2017unified}. Here, we consider a variant close to the Leave Out COvariates (LOCO) method \parencite{lei2018distribution}, which we apply to PGB to evaluate its ability to identify the most important components of $\boldsymbol{x}$ to predict the distribution of $Y$. More specifically, we generate 2000 observations from the \enquote{irrelevant covariates} and \enquote{non-sparse data} models presented in Section~\ref{sec:expsim}, with $d=15$ in both cases. We train a PGB model on these data, which we denote $\hat{f}_{Y\mid \cdot}$, as well as on the 15 datasets obtained by removing one column from the feature matrix. We denote these 15 models by $\hat{f}_{Y\mid \cdot}^{(-j)}$, for $j\in\{1,\ldots,15\}$. The importance score for feature $j$ is given by the integrated squared difference (ISD) between $\hat{f}_{Y\mid \cdot}$ and $\hat{f}_{Y\mid \cdot}^{(-j)}$, defined analogously to Equation~\eqref{eq:ise} as:
\begin{equation*}
\operatorname{ISD}(\hat{f}_{Y\mid\cdot},\hat{f}_{Y\mid \cdot}^{(-j)}) = \iint(\hat{f}_{Y\mid \boldsymbol{x}}(y)-\hat{f}_{Y\mid \boldsymbol{x}}^{(-j)}(y))^2dydF_X(\boldsymbol{x}).
\end{equation*}
The ISD is approximated as the ISE in the previous experiment (see Section~\ref{sec:expsim}). We follow the same process to obtain importance scores from DRF models. The results of this experiment are presented in Figure~\ref{fig:importance}.

\begin{figure}[!htbp]
\centering
\includegraphics[width=0.85\linewidth]{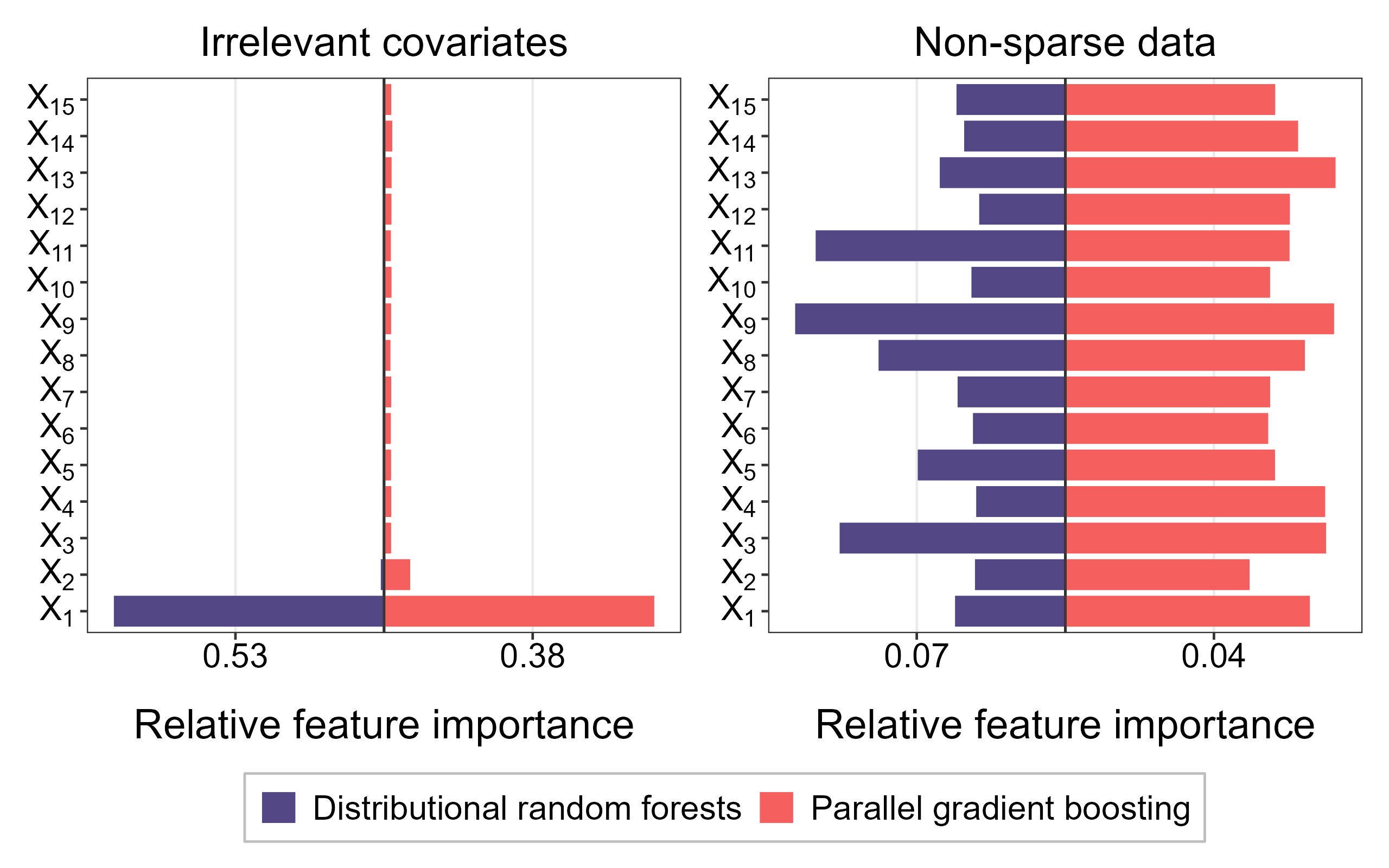}
\caption{\label{fig:importance}Importance scores for data simulated according to the irrelevant covariates (left) and non-sparse data (right) models. The scores are normalized to sum to one.}
\end{figure}

In the irrelevant covariates setting, $X_1$ is identified as the most important predictor for both PGB and DRF. However, only with PGB is $X_2$ identified as contributing significantly more than the other features to the prediction. This suggests that PGB, contrary to DRF, effectively captures the dependence of $Y$ on $X$ beyond the conditional mean. With non-sparse data, both methods attribute similar importance scores to all features, but the scores are more homogeneous in the case of PGB, which is desirable in this case.

\subsection{Performance on real data}

We now evaluate the performance of PGB compared to the competing estimators on real (medical) datasets, that were chosen for their diversity in terms of dimensionality, structure of predictors and distributions of the response variable. The characteristics of these datasets are given in Table~\ref{tab:data}. In order to conduct the experiment, they underwent a few pre-processing steps, including the removal of observations with missing values for the dependent variable, re-coding of some features, removal of constant or nearly constant features, and deletion of discrete components in numeric features, if applicable. The dependent variables were also standardized by their mean and variance to facilitate comparisons of the prediction errors. All of this explains slight differences with other versions of these datasets that can be found in public repositories.

\begin{table}[!htbp]
  \centering
  \small
  \begin{tabular}{|>{\centering\arraybackslash}m{3.85cm}|c|c|c|c|>{\centering\arraybackslash}m{2.3cm}|}
    \hline
    \textbf{Name (source)} & $\boldsymbol{n}$ & $\boldsymbol{d_{\operatorname{num}}}$ & $\boldsymbol{d_{\operatorname{cat}}}$ & $\boldsymbol{p_{\operatorname{missing}}}$ & $\boldsymbol{y}$ \\ \hline
    \textbf{Prostate}\par\parencite{friendly2020package} & 97 & 8 & 1 & 0\% & log PSA \\ \hline
    \textbf{Supraclavicular}\par\parencite{medicaldata2022} & 102 & 12 & 4 & 0.4\% & log onset sensory \\ \hline
    \textbf{Blood storage}\par\parencite{medicaldata2022} & 315 & 5 & 14 & 0.6\% & log time to recurrence \\ \hline
    \textbf{Alzheimer's disease}\par\parencite{applied2018} & 333 & 128 & 3 & 0\% & A$\beta$42 \\ \hline
    \textbf{Bone density}\par\parencite{loondata2025} & 485 & 2 & 2 & 0.1\% & Relative spinal bone mineral density \\ \hline
    \textbf{Obstetrics}\par\parencite{medicaldata2022} & 823 & 108 & 58 & 20.8\% & GA at outcome \\ \hline
    \textbf{COVID}\par\parencite{medicaldata2022} & 870 & 5 & 8 & 8.4\% & CT result \\ \hline
    \textbf{Nutrition}\par\parencite{causaldata2024} & 1566 & 25 & 36 & 6.2\% & Weight change \\ \hline
    \textbf{Peak VO2}\par\parencite{rfsrc2025} & 2231 & 14 & 27 & 0\% & Peak VO2 \\ \hline
    \textbf{Parkinson}\par\parencite{Athanasios_Tsanas2009-gr} & 5875 & 20 & 1 & 0\% & Total UPDRS \\ \hline
  \end{tabular}
  \caption{Characteristics of the medical datasets used in this work. In the table, $n$ refers to the number of observations, $d_{\operatorname{num}}$ to the number of numeric variables, $d_{\operatorname{cat}}$ to the number of categorical variables, $p_{\operatorname{missing}}$ to the proportion of missing values, and $y$ to the name of the dependent variable chosen for each dataset.}
  \label{tab:data}
\end{table}

In this experiment, we randomly split each dataset into a training set and a test set of the same size. The training set is used to train the models (including the optimization of the hyperparameters). The test set is used to estimate the performance of each model in terms of expected CRPS. The experiment is run 50 times for each dataset and method, and the results are combined to obtain mean test errors and computation times, as well as the corresponding standard deviations. The results of the experiment are given in Tables~\ref{tab:rescrps} and \ref{tab:restime}.

\begin{table}[!htbp]
  \centering
  \small
    \begin{tabular}{|c|c|c|c|c|}
    \hline
    \textbf{Dataset} & \textbf{DRF} & \textbf{FlexCode} & \textbf{LinCDE} & \textbf{PGB} \\ \hline
    Prostate
    & \makecell{0.561\\(0.042)}
    & \makecell{0.468\\(0.067)}
    & \cellcolor{yellow!15}\colorbox{yellow!15}{\makecell{0.454\\(0.045)}}
    & \cellcolor{yellow!40}\colorbox{yellow!40}{\makecell{\textbf{0.400}\\\textbf{(0.040)}}} \\ \hline
    Supraclavicular
    & \makecell{0.568\\(0.036)}
    & \makecell{0.580\\(0.049)}
    & \cellcolor{yellow!40}\colorbox{yellow!40}{\makecell{\textbf{0.516}\\\textbf{(0.043)}}}
    & \cellcolor{yellow!15}\colorbox{yellow!15}{\makecell{0.521\\(0.040)}} \\ \hline
    Blood storage
    & \cellcolor{yellow!15}\colorbox{yellow!15}{\makecell{0.554\\(0.021)}}
    & \makecell{0.569\\(0.029)}
    & \cellcolor{yellow!40}\colorbox{yellow!40}{\makecell{\textbf{0.552}\\\textbf{(0.026)}}}
    & \makecell{0.559\\(0.023)} \\ \hline
    Alzheimer's disease
    & \makecell{0.548\\(0.019)}
    & \cellcolor{yellow!40}\colorbox{yellow!40}{\makecell{\textbf{0.486}\\\textbf{(0.029)}}}
    & \makecell{0.502\\(0.024)}
    & \cellcolor{yellow!15}\colorbox{yellow!15}{\makecell{0.500\\(0.024)}} \\ \hline
    Bone density
    & \cellcolor{yellow!15}\colorbox{yellow!15}{\makecell{0.430\\(0.019)}}
    & \makecell{0.460\\(0.021)}
    & \makecell{0.435\\(0.021)}
    & \cellcolor{yellow!40}\colorbox{yellow!40}{\makecell{\textbf{0.427}\\\textbf{(0.021)}}} \\ \hline
    Obstetrics
    & \makecell{0.247\\(0.021)}
    & \cellcolor{yellow!15}\colorbox{yellow!15}{\makecell{0.185\\(0.007)}}
    & \makecell{0.190\\(0.013)}
    & \cellcolor{yellow!40}\colorbox{yellow!40}{\makecell{\textbf{0.174}\\\textbf{(0.007)}}} \\ \hline
    COVID
    & \cellcolor{yellow!40}\colorbox{yellow!40}{\makecell{\textbf{0.553}\\\textbf{(0.009)}}}
    & \makecell{0.556\\(0.016)}
    & \makecell{0.557\\(0.012)}
    & \cellcolor{yellow!15}\colorbox{yellow!15}{\makecell{0.555\\(0.011)}} \\ \hline
    Nutrition
    & \makecell{0.511\\(0.016)}
    & \cellcolor{yellow!40}\colorbox{yellow!40}{\makecell{\textbf{0.508}\\\textbf{(0.016)}}}
    & \makecell{0.509\\(0.017)}
    & \cellcolor{yellow!15}\colorbox{yellow!15}{\makecell{0.509\\(0.015)}}\\ \hline
    Peak VO2
    & \cellcolor{yellow!15}\colorbox{yellow!15}{\makecell{0.267\\(0.005)}}
    & \makecell{0.347\\(0.006)}
    & \makecell{0.268\\(0.005)}
    & \cellcolor{yellow!40}\colorbox{yellow!40}{\makecell{\textbf{0.242}\\\textbf{(0.006)}}} \\ \hline
    Parkinson
    & \cellcolor{yellow!15}\colorbox{yellow!15}{\makecell{0.075\\(0.002)}}
    & \makecell{0.291\\(0.006)}
    & \makecell{0.171\\(0.013)}
    & \cellcolor{yellow!40}\colorbox{yellow!40}{\makecell{\textbf{0.055}\\\textbf{(0.001)}}} \\ \hline
    \end{tabular}
  \caption{Test errors of the competing estimation methods applied to the datasets listed in Table~\ref{tab:data}. The loss function used is the continuous ranked probability score. Numbers in parentheses are the standard deviations derived from repeated trials of the experiment.}
  \label{tab:rescrps}
\end{table}

\begin{table}[!htbp]
  \centering
  \small
    \begin{tabular}{|c|c|c|c|c|}
    \hline
    \textbf{Dataset} & \textbf{DRF} & \textbf{FlexCode} & \textbf{LinCDE} & \textbf{PGB} \\ \hline
    Prostate
    & \cellcolor{yellow!15}\colorbox{yellow!15}{\makecell{0.067\\(0.010)}}
    & \makecell{0.948\\(0.011)}
    & \makecell{0.255\\(0.035)}
    & \cellcolor{yellow!40}\colorbox{yellow!40}{\makecell{\textbf{0.059}\\\textbf{(0.013)}}} \\ \hline
    Supraclavicular
    & \cellcolor{yellow!15}\colorbox{yellow!15}{\makecell{0.070\\(0.010)}}
    & \makecell{1.011\\(0.016)}
    & \makecell{0.269\\(0.045)}
    & \cellcolor{yellow!40}\colorbox{yellow!40}{\makecell{\textbf{0.069}\\\textbf{(0.011)}}} \\ \hline
    Blood storage
    & \makecell{1.947\\(0.066)}
    & \makecell{1.211\\(0.041)}
    & \cellcolor{yellow!15}\colorbox{yellow!15}{\makecell{0.530\\(0.137)}}
    & \cellcolor{yellow!40}\colorbox{yellow!40}{\makecell{\textbf{0.184}\\\textbf{(0.015)}}} \\ \hline
    Alzheimer's disease
    & \makecell{2.531\\(0.050)}
    & \makecell{7.280\\(0.287)}
    & \cellcolor{yellow!15}\colorbox{yellow!15}{\makecell{1.728\\(0.109)}}
    & \cellcolor{yellow!40}\colorbox{yellow!40}{\makecell{\textbf{0.624}\\\textbf{(0.046)}}} \\ \hline
    Bone density
    & \makecell{1.393\\(0.014)}
    & \makecell{1.272\\(0.051)}
    & \cellcolor{yellow!15}\colorbox{yellow!15}{\makecell{0.503\\(0.050)}}
    & \cellcolor{yellow!40}\colorbox{yellow!40}{\makecell{\textbf{0.350}\\\textbf{(0.016)}}} \\ \hline
    Obstetrics
    & \makecell{13.591\\(0.213)}
    & \makecell{15.794\\(1.360)}
    & \cellcolor{yellow!15}\colorbox{yellow!15}{\makecell{8.353\\(0.507)}}
    & \cellcolor{yellow!40}\colorbox{yellow!40}{\makecell{\textbf{2.846}\\\textbf{(0.237)}}} \\ \hline
    COVID
    & \makecell{9.243\\(0.187)}
    & \makecell{1.711\\(0.063)}
    & \cellcolor{yellow!15}\colorbox{yellow!15}{\makecell{1.009\\(0.190)}}
    & \cellcolor{yellow!40}\colorbox{yellow!40}{\makecell{\textbf{0.806}\\\textbf{(0.030)}}} \\ \hline
    Nutrition
    & \makecell{25.068\\(0.430)}
    & \cellcolor{yellow!15}\colorbox{yellow!15}{\makecell{3.092\\(0.162)}}
    & \makecell{5.315\\(0.514)}
    & \cellcolor{yellow!40}\colorbox{yellow!40}{\makecell{\textbf{2.162}\\\textbf{(0.064)}}} \\ \hline
    Peak VO2
    & \makecell{31.719\\(0.241)}
    & \cellcolor{yellow!40}\colorbox{yellow!40}{\makecell{\textbf{3.053}\\\textbf{(0.150)}}}
    & \makecell{6.637\\(0.357)}
    & \cellcolor{yellow!15}\colorbox{yellow!15}{\makecell{4.216\\(0.129)}} \\ \hline
    Parkinson
    & \makecell{85.988\\(0.681)}
    & \cellcolor{yellow!40}\colorbox{yellow!40}{\makecell{\textbf{6.423}\\\textbf{(0.099)}}}
    & \cellcolor{yellow!15}\colorbox{yellow!15}{\makecell{15.239\\(0.172)}}
    & \makecell{46.428\\(3.422)} \\ \hline
    \end{tabular}
  \caption{Computation times (in seconds) of the competing estimation methods applied to the datasets listed in Table~\ref{tab:data}. Numbers in parentheses are the standard deviations derived from repeated trials of the experiment.}
  \label{tab:restime}
\end{table}

Overall, PGB performs better than the competing methods in this experiment, either in terms of test error or computation times. More specifically, PGB provides the best or second best estimates for 9 out of 10 datasets, and is also the fastest or second fastest method for 9 out of 10 datasets. These datasets are diverse in terms of number of observations, dimensionality, and patterns of missing data, but PGB was used with the same hyperparameters and early stopping strategy in each case. This suggests that PGB can offer competitive performance without requiring extensive hyperparameter tuning. In many situations, using the same default hyperparameters as in this experiment should therefore provide satisfactory results. With large datasets, it is even possible that giving up the cross-validation strategy in favor of the more classical train-validation split would yield similar test errors but significantly lower training times. We do not explore this approach here, the strategy to tune the hyperparameters being kept unique in order to simplify the interpretation of the experiments.

\section{Discussion}\label{sec:discussion}

In this paper, we introduce PGB, a general approach for efficient multi-output gradient boosting that does not require specialized base learners and can therefore be easily integrated into any existing boosting framework. The approach finds a natural application in estimating conditional distributions via multiple quantile regression. In this setting, it is amenable to implementations that provide high performance in both the quality of the estimates and computation times. Its nonparametric nature and natural handling of mixed and missing covariates allow it to be applied to any tabular dataset. The flexibility of the method is illustrated by its competitive performance in various data generating scenarios, ranging from irrelevant features to dense covariates.

Compared to other nonparametric and semiparametric estimators of conditional distributions, PGB often performs better. Its relative performance compared to historical estimation methods, such as kernel-based ones \parencite{rosenblatt1969conditional}, or newer methods such as generative ones \parencite{reisach2025transforming, chin2026generative}, is not evaluated in this paper. The reason is that, in our view, such methods do not meet the requirements of an \enquote{off-the-shelf} procedure, which is the natural positioning of PGB in the statistical toolbox. For example, kernel conditional density estimators are hardly scalable to datasets of moderate to large size \parencite{zhao2025adaptive}. Other methods admit many hyperparameters that may be hard to tune. We therefore do not consider these methods to be relevant competitors to PGB for conditional distribution estimation.

A natural line of future research on PGB is to study its convergence rate and Bayes-risk consistency, as was done more than 20 years ago for single-output gradient boosting \parencite{lugosi2004bayes}. A related open question is what properties the regression targets should possess for the algorithm to perform well. In multiple quantile regression, for example, the targets are related through a common conditional probability distribution, which allows information to be shared between quantile levels. Closely related to this phenomenon is the fact that, in our implementation, using denser grids of quantiles seldom changes the output of the algorithm beyond a certain point. Understanding the conditions under which this property holds is important to find new applications of PGB, especially for more general multi-task learning problems. Conversely, it would be desirable to derive variants of the algorithm with better properties, either in terms of generalization, convergence rates, or adaptivity to other types of task relatedness. An especially tempting line of research in this regard would be to study the application of PGB to multiclass classification, and, if necessary, adapt the algorithm to better suit this class of problems. A natural starting point for that could be to adapt existing proposals that are specific to decision tree boosting \parencite{joly2019gradient, iosipoi2022sketchboost}, and study their theoretical properties within the more general framework of PGB.

\newpage
\begin{appendices}

\section{Proofs of theoretical results}\label{sec:proofs}

\subsection{Proof of Theorem~\ref{th:gber}}

To prove Theorem~\ref{th:gber}, we will need intermediate technical results, which are given hereafter along with their proofs. In all of these, we denote
\begin{equation*}
\kappa_t=-\sum_{i=1}^n\langle\nabla_2L_i(g_{t-1}),\varphi_{\boldsymbol{\theta}_t}(\boldsymbol{x}_i)\rangle
\end{equation*}
for $t\geq1$, and let $K$ be the smoothness constant mentioned in Assumption~\ref{assum:smooth}. In all of the following lemmas, we make the same assumptions as in Theorem~\ref{th:gber}.

\begin{lemma}\label{lemma:extract}
Under the choice of step sizes stated in Theorem~\ref{th:gber}, there exists a strictly increasing extraction function $\sigma:\mathbb{N}\rightarrow\mathbb{N}$ such that $b_{\sigma(s)+1}=\frac{\kappa_{\sigma(s)+1}}{K}$ for all $s\in\mathbb{N}$, and $\lim_{T\rightarrow\infty}\sum_{s=1}^T\frac{1}{\sigma(s)}=\infty$.
\end{lemma}
\begin{proof}
After iteration $t\geq0$, by the definition of $K$-smoothness (Assumption~\ref{assum:smooth}), we have for each $i\in\{1,\ldots,n\}$, $b\in\mathbb{R}$ and $\boldsymbol{\theta}\in\Theta$:
\begin{equation*}
L_i(g_t+b\varphi_{\boldsymbol{\theta}}) \leq L_i(g_t) + \langle\nabla_2L_i(g_t), b\varphi_{\boldsymbol{\theta}}(\boldsymbol{x}_i)\rangle+\frac{K}{2}\Vert b\varphi_{\boldsymbol{\theta}}(\boldsymbol{x}_i) \Vert_2^2.
\end{equation*}
Hence, by summing from $i=1$ to $i=n$,
\begin{equation}\label{eq:ineqsmooth}
\sum_{i=1}^n L_i(g_t+b\varphi_{\boldsymbol{\theta}}) \leq \sum_{i=1}^n (L_i(g_t) + \langle\nabla_2L_i(g_t), b\varphi_{\boldsymbol{\theta}}(\boldsymbol{x}_i)\rangle+\frac{K}{2}\Vert b\varphi_{\boldsymbol{\theta}}(\boldsymbol{x}_i) \Vert_2^2).
\end{equation}
Note that by Assumption~\ref{assum:bound},
\begin{equation}\label{eq:predbound}
\sum_{i=1}^n \frac{K}{2}\Vert b\varphi_{\boldsymbol{\theta}}(\boldsymbol{x}_i) \Vert_2^2=\frac{Kb^2}{2}\sum_{i=1}^n \Vert \varphi_{\boldsymbol{\theta}}(\boldsymbol{x}_i) \Vert_2^2 \leq \frac{Kb^2}{2}.
\end{equation}
Hence (combining Equations~\eqref{eq:ineqsmooth} and \eqref{eq:predbound} above):
\begin{equation}\label{eq:empdec}
\hat{\mathcal{R}}_n(g_t+b\varphi_{\boldsymbol{\theta}}) \leq  \hat{\mathcal{R}}_n(g_t) + \sum_{i=1}^n\langle\nabla_2L_i(g_t), b\varphi_{\boldsymbol{\theta}}(\boldsymbol{x}_i)\rangle+\frac{Kb^2}{2}.
\end{equation}
By Assumption~\ref{assum:sym}, note that $\kappa_{t+1}\geq0$. Indeed, suppose the following:
\begin{equation*}
\kappa_{t+1} = -\sum_{i=1}^n\langle\nabla_2L_i(g_t),\varphi_{\boldsymbol{\theta}_{t+1}}(\boldsymbol{x}_i)\rangle < 0.
\end{equation*}
We would then have $\sum_{i=1}^n\langle\nabla_2L_i(g_t),\varphi_{\boldsymbol{\theta}_{t+1}}(\boldsymbol{x}_i)\rangle > 0$, and using the symmetric parameter $\bar{\boldsymbol{\theta}}_{t+1}$, we would have:
\begin{align*}
\sum_{i=1}^n\langle\nabla_2L_i(g_t),\varphi_{\bar{\boldsymbol{\theta}}_{t+1}}(\boldsymbol{x}_i)\rangle &= - \sum_{i=1}^n\langle\nabla_2L_i(g_t),\varphi_{\boldsymbol{\theta}_{t+1}}(\boldsymbol{x}_i)\rangle \\
&< \sum_{i=1}^n\langle\nabla_2L_i(g_t),\varphi_{\boldsymbol{\theta}_{t+1}}(\boldsymbol{x}_i)\rangle
\end{align*}
which would contradict the minimization principle stated in Equation~\eqref{eq:paramit}.

The rule for choosing step sizes given in Theorem~\ref{th:gber} therefore leads to $0 \leq b_{t+1}\leq \frac{\kappa_{t+1}}{K}$, hence $\kappa_{t+1}\geq Kb_{t+1}$ and $\kappa_{t+1}b_{t+1}\geq Kb_{t+1}^2$. By inserting this in Equation~\eqref{eq:empdec}, we get:
\begin{equation*}
\hat{\mathcal{R}}_n(g_t) - \hat{\mathcal{R}}_n(g_{t+1}) \geq b_{t+1}\kappa_{t+1}-\frac{Kb_{t+1}^2}{2}\geq Kb_{t+1}^2-\frac{Kb_{t+1}^2}{2}=\frac{Kb_{t+1}^2}{2}.
\end{equation*}
As this is true for all $t$'s, we can write (using telescopic sums):
\begin{equation*}
\sum_{s=0}^{S-1} \frac{Kb_{s+1}^2}{2} \leq \sum_{s=0}^{S-1}(\hat{\mathcal{R}}_n(g_s) - \hat{\mathcal{R}}_n(g_{s+1}))=\hat{\mathcal{R}}_n(g_0)-\hat{\mathcal{R}}_n(g_S)
\end{equation*}
for each $S\geq1$. This leads to:
\begin{equation*}
\sum_{s=0}^{S-1}b_{s+1}^2 \leq \frac{2(\hat{\mathcal{R}}_n(g_0)-\hat{\mathcal{R}}_n(g_S))}{K}\leq \frac{2(\hat{\mathcal{R}}_n(g_0)-\hat{\mathcal{R}}_n(g_\star))}{K}.
\end{equation*}
This implies that $\lim_{S\rightarrow\infty}\sum_{s=0}^{S}b_{s+1}^2<\infty$. It follows that the set
\begin{equation*}
A=\{s\geq 0\mid b_{s+1}=\frac{\kappa_{s+1}}{K}\}
\end{equation*}
is infinite. Indeed, if it were not, there would be an index $\tau$ such that $b_s=\frac{1}{\sqrt{s}}$ for each $s>\tau$. We would then have $\lim_{S\rightarrow\infty}\sum_{s=\tau}^{S}b_{s+1}^2=\infty$, which contradicts the previous point. Let then $\sigma:\mathbb{N}\rightarrow\mathbb{N}$ be the strictly increasing enumeration of $A$. Clearly, $\sigma$ satisfies the first criterion of the lemma, i.e. $b_{\sigma(s)+1}=\frac{\kappa_{\sigma(s)+1}}{K}$ for all $s\in\mathbb{N}$. Let us now show that $\lim_{T\rightarrow\infty}\sum_{s=1}^T\frac{1}{\sigma(s)}=\infty$. Let $E_T=A\cap\{0,\ldots,T\}$ and $E_T^c=\{0,\ldots,T\}\setminus E_T$ for $T\geq0$. For each such $T$, we have
\begin{align*}
\sum_{s=0}^T\frac{1}{s+1}&=\sum_{s\in E_T}\frac{1}{s+1}+\sum_{s\in E_T^c}\frac{1}{s+1} \\
&= \sum_{s\in E_T}\frac{1}{s+1}+\sum_{s\in E_T^c}b_{s+1}^2.
\end{align*}
The left-hand side of the latter equation diverges, while we have already established that $\lim_{T\rightarrow\infty}\sum_{s\in E_T^c}b_{s+1}^2<\infty$. Therefore, we have
\begin{equation*}
\lim_{T\rightarrow\infty}\sum_{s\in E_T}\frac{1}{s+1}=\infty,
\end{equation*}
which implies that $\lim_{S\rightarrow\infty}\sum_{s=1}^S\frac{1}{\sigma(s)+1}=\infty$. As $\frac{1}{\sigma(s)+1}\leq\frac{1}{\sigma(s)}$ for all $s\geq 1$, we finally have:
\begin{equation*}
\lim_{T\rightarrow\infty}\sum_{s=1}^T\frac{1}{\sigma(s)}=\infty.
\end{equation*}
\end{proof}

\begin{lemma}\label{lemma:bound}
Writing
\begin{equation*}
g_\star=\sum_{s=1}^{T_\star}b^\star_s\varphi_{\boldsymbol{\theta}^\star_s}
\end{equation*}
for some $T_\star\in\mathbb{N}^*$, where each $b^\star_s\in\mathbb{R}^+$ (the requirement of positivity is without loss of generality because of Assumption~\ref{assum:sym}) and $\boldsymbol{\theta}^\star_s\in\Theta$, let
\begin{equation*}
B_\star=\sum_{s=1}^{T_\star}b^\star_s.
\end{equation*}
For each $t>0$, we have:
\begin{equation*}
\kappa_{t+1}\geq\frac{\hat{\mathcal{R}}_n(g_t)-\hat{\mathcal{R}}_n(g_\star)}{B_\star + 2\sqrt{t}}.
\end{equation*}
\end{lemma}
\begin{proof}
Let $t>0$. As of the rule given in Theorem~\ref{th:gber} for choosing step sizes, we have $b_s\leq\frac{1}{\sqrt{s}}$ for all $s\leq t$, hence $0\leq \sum_{s=1}^tb_s<2\sqrt{t}$.
Now, note that by Assumption~\ref{assum:sym}, $-g_t$ can be written as $\sum_{s=1}^tb_s\varphi_{\bar{\boldsymbol{\theta}}_s}$ with $b_s>0$, $s\in\{1,\ldots,t\}$. Hence, $g_\star-g_t$ can be written as :
\begin{equation*}
g_\star-g_t=(\sum_{s=1}^{T_\star}b^\star_s\varphi_{\boldsymbol{\theta}^\star_s})+(\sum_{s=1}^tb_s\varphi_{\bar{\boldsymbol{\theta}}_s}),
\end{equation*}
where all coefficients are positive and $0\leq(\sum_{s=1}^{T_\star}b^\star_s) + (\sum_{s=1}^tb_s)<B_\star+ 2\sqrt{t}$. For each $i\in\{1,\ldots,n\}$, we therefore have:
\begin{align*}
\langle\nabla_2L_i(g_t), g_\star-g_t\rangle &= \langle\nabla_2L_i(g_t), \sum_{s=1}^{T_\star}b^\star_s\varphi_{\boldsymbol{\theta}^\star_s}\rangle + \langle\nabla_2L_i(g_t), \sum_{s=1}^tb_s\varphi_{\bar{\boldsymbol{\theta}}_s}\rangle \\
&= \sum_{s=1}^{T_\star}b^\star_s\langle\nabla_2L_i(g_t), \varphi_{\boldsymbol{\theta}^\star_s}\rangle + \sum_{s=1}^tb_s\langle\nabla_2L_i(g_t), \varphi_{\bar{\boldsymbol{\theta}}_s}\rangle
\end{align*}
Hence by summing from $i=1$ to $i=n$,
\begin{align*}
\sum_{i=1}^n\langle\nabla_2L_i(g_t), g_\star(\boldsymbol{x}_i)-g_t(\boldsymbol{x}_i)\rangle = {} &(\sum_{s=1}^{T_\star}b^\star_s\sum_{i=1}^n\langle\nabla_2L_i(g_t),\varphi_{\boldsymbol{\theta}^\star_s}(\boldsymbol{x}_i)\rangle) \\
& + (\sum_{s=1}^tb_s\sum_{i=1}^n\langle\nabla_2L_i(g_t),\varphi_{\bar{\boldsymbol{\theta}}_s}(\boldsymbol{x}_i)\rangle).
\end{align*}
Let $\lambda_t$ denote the minimum of
\begin{equation*}
\bigl\{
\min_{1\leq s\leq T_\star}\sum_{i=1}^n\langle\nabla_2L_i(g_t),\varphi_{\boldsymbol{\theta}^\star_s}(\boldsymbol{x}_i)\rangle,
\min_{1\leq s\leq t}\sum_{i=1}^n\langle\nabla_2L_i(g_t),\varphi_{\bar{\boldsymbol{\theta}}_s}(\boldsymbol{x}_i)\rangle\bigr\}.
\end{equation*}
By dividing by $C=B_\star+\sum_{s=1}^tb_s$, we get:
\begin{equation}\label{eq:infbound}
\lambda_t\leq \frac{1}{C}\sum_{i=1}^n\langle\nabla_2L_i(g_t), g_\star(\boldsymbol{x}_i)-g_t(\boldsymbol{x}_i)\rangle.
\end{equation}
Using the convexity of $L$ with respect to its second argument, we have for all $i$'s:
\begin{equation*}
L_i(g_\star(\boldsymbol{x}_i))\geq L_i(g_t(\boldsymbol{x}_i))+\langle 
\nabla_2L_i(g_t),g_\star(\boldsymbol{x}_i)-g_t(\boldsymbol{x}_i)\rangle.
\end{equation*}
Hence by summing over all indices:
\begin{equation}\label{eq:supbound}
\hat{\mathcal{R}}_n(g_\star) \geq \hat{\mathcal{R}}_n(g_t)+\sum_{i=1}^n\langle\nabla_2L_i(g_t),g_\star(\boldsymbol{x}_i)-g_t(\boldsymbol{x}_i)\rangle.
\end{equation}
Combining \eqref{eq:infbound} and \eqref{eq:supbound} leads to:
\begin{equation*}
\lambda_t\leq \frac{1}{C}\sum_{i=1}^n\langle\nabla_2L_i(g_t), g_\star(\boldsymbol{x}_i)-g_t(\boldsymbol{x}_i)\rangle\leq\frac{1}{C}(\hat{\mathcal{R}}_n(g_\star)-\hat{\mathcal{R}}_n(g_t)).
\end{equation*}
And because $\boldsymbol{\theta}_{t+1}$ is chosen according to Equation~\eqref{eq:paramit}, we have:
\begin{equation}\label{eq:keyineq}
-\kappa_{t+1}\leq\frac{1}{C}\sum_{i=1}^n\langle\nabla_2L_i(g_t), g_\star(\boldsymbol{x}_i)-g_t(\boldsymbol{x}_i)\rangle\leq\frac{1}{C}(\hat{\mathcal{R}}_n(g_\star)-\hat{\mathcal{R}}_n(g_t)).
\end{equation}
Hence
\begin{equation}\label{eq:kappa}
\kappa_{t+1}\geq\frac{\hat{\mathcal{R}}_n(g_t)-\hat{\mathcal{R}}_n(g_\star)}{C}\geq\frac{\hat{\mathcal{R}}_n(g_t)-\hat{\mathcal{R}}_n(g_\star)}{B_\star + 2\sqrt{t}}.
\end{equation}
The result follows immediately.
\end{proof}

We can now proceed with the proof of Theorem~\ref{th:gber}.

\begin{proof}[Proof of Theorem~\ref{th:gber}]
For all $t\geq 0$, by the definition of $K$-smoothness, we have (see the first developments in the proof of Lemma~\ref{lemma:extract}):
\begin{equation*}
\hat{\mathcal{R}}_n(g_{t+1}) \leq  \hat{\mathcal{R}}_n(g_t)-b_{t+1}\kappa_{t+1}
+\frac{Kb_{t+1}^2}{2}.
\end{equation*}
As $b_{t+1}\leq\frac{\kappa_{t+1}}{K}$, we have $\frac{Kb_{t+1}^2}{2}\leq\frac{b_{t+1}\kappa_{t+1}}{2}$, hence
\begin{equation*}
\hat{\mathcal{R}}_n(g_{t+1}) \leq  \hat{\mathcal{R}}_n(g_t) - \frac{b_{t+1}\kappa_{t+1}}{2}.
\end{equation*}
Let $\Delta_t=\hat{\mathcal{R}}_n(g_{t})-\hat{\mathcal{R}}_n(g_\star)$. The latter inequality gives:
\begin{equation}\label{eq:delta}
    \Delta_{t+1} \leq  \Delta_t - \frac{b_{t+1}\kappa_{t+1}}{2}.
\end{equation}
The sequence $(\Delta_t)_{t\in\mathbb{N}}$ is bounded from below by $0$. As it is also decreasing, it is convergent. Let $l$ be the limit of $(\Delta_t)_{t\in\mathbb{N}}$. Then $\lim_{t\rightarrow\infty}\Delta_{\sigma(t)}=l$, where $\sigma$ is the extraction function introduced in Lemma~\ref{lemma:extract}. Now, by the definition of this extraction function, we have for all $t$:
\begin{equation*}
\frac{b_{\sigma(t)+1}\kappa_{\sigma(t)+1}}{2}=\frac{\kappa_{\sigma(t)+1}^2}{2K}.
\end{equation*}
By Lemma~\ref{lemma:bound}, $\kappa_{\sigma(t)+1}\geq\frac{\Delta_{\sigma(t)}}{B_\star + 2\sqrt{\sigma(t)}}$. Hence:
\begin{equation*}
\frac{\kappa_{\sigma(t)+1}^2}{2K}\geq\frac{\Delta_{\sigma(t)}^2}{2K(B_\star + 2\sqrt{\sigma(t)})^2}.
\end{equation*}
Combining this last result with Equation~\eqref{eq:delta} provides:
\begin{equation*}
\Delta_{\sigma(t+1)} \leq \Delta_{\sigma(t)+1} \leq  \Delta_{\sigma(t)} - \frac{\Delta_{\sigma(t)}^2}{2K(B_\star + 2\sqrt{\sigma(t)})^2}.
\end{equation*}
Let $(c,t_0)\in(\mathbb{R}^+,\mathbb{N}^*)$ be any constants such that $2K(B_\star+2\sqrt{\sigma(t)})^2\leq c\sigma(t)$ holds for each $t\geq t_0$ (such constants always exist, as shown by the elementary analysis of $(B_\star+2\sqrt{\sigma(t)})^2$ seen as a function of $\sigma(t)$). From now on, we will indeed admit that $t\geq t_0$. In this case, we have:
\begin{equation*}
\Delta_{\sigma(t+1)} \leq  \Delta_{\sigma(t)} - \frac{\Delta_{\sigma(t)}^2}{c\sigma(t)}.
\end{equation*}
By elementary results on the geometric series, this implies: 
\begin{equation*}
\frac{1}{\Delta_{\sigma(t+1)}} - \frac{1}{\Delta_{\sigma(t)}} \geq  \frac{1}{c\sigma(t)}.
\end{equation*}
Using telescopic sums, we get
\begin{equation*}
\frac{1}{\Delta_{\sigma(T+1)}} \geq \frac{1}{\Delta_{\sigma(t_0)}} + \frac{1}{c}\sum_{t=t_0}^T\frac{1}{\sigma(t)}.
\end{equation*}
By Lemma~\ref{lemma:extract}, the right-hand part of this inequality diverges to $\infty$ when $T\rightarrow\infty$, hence we also have $\lim_{t\rightarrow\infty}\frac{1}{\Delta_{\sigma(t)}}=\infty$, so $l=0$, which is the desired result.
\end{proof}

\subsection{Proof of Theorem~\ref{th:gber2}}

We now turn to the proof of Theorem~\ref{th:gber2}. Compared to that of Theorem~\ref{th:gber}, it is simpler thanks to Assumption~\ref{assum:compact}, which helps linking the empirical risk with the alignment between the predictions and the gradient values. In the following, we use the same definition for $\kappa_t$ as in the previous subsection.

\begin{proof}[Proof of Theorem~\ref{th:gber2}]
As in the proof of Lemma~\ref{lemma:extract}, we begin by remarking that by the definition of K-smoothness, we have for every $t\geq 0$ and $b\in\mathbb{R}$,
\begin{equation*}
\hat{\mathcal{R}}_n(g_t+b\varphi_{\boldsymbol{\theta}_{t+1}}) \leq  \hat{\mathcal{R}}_n(g_t) + \sum_{i=1}^n\langle\nabla_2L_i(g_t), b\varphi_{\boldsymbol{\theta}_{t+1}}(\boldsymbol{x}_i)\rangle+\frac{Kb^2}{2}.
\end{equation*}
And because $\eta_{t+1}\in\argmin_{\eta\in\mathbb{R}^+} \hat{\mathcal{R}}_n(g_t+\eta\varphi_{\boldsymbol{\theta}_{t+1}})$, we have:
\begin{align*}
\hat{\mathcal{R}}_n(g_t+\eta_{t+1}\varphi_{\boldsymbol{\theta}_{t+1}}) &\leq \hat{\mathcal{R}}_n(g_t) + b\sum_{i=1}^n\langle\nabla_2L_i(g_t), \varphi_{\boldsymbol{\theta}_{t+1}}(\boldsymbol{x}_i)\rangle+\frac{Kb^2}{2} \\
&=\hat{\mathcal{R}}_n(g_t) - b\kappa_{t+1}+\frac{Kb^2}{2}.
\end{align*}
In particular, taking $b=\frac{\kappa_{t+1}}{K}$, we get:
\begin{equation}\label{eq:ineqkappa}
\hat{\mathcal{R}}_n(g_t+\eta_{t+1}\varphi_{\boldsymbol{\theta}_{t+1}}) \leq \hat{\mathcal{R}}_n(g_t) - \frac{\kappa_{t+1}^2}{2K}.
\end{equation}
Now, remark that by the convexity of the empirical risk, we can write:
\begin{align}\label{eq:ineqls}
\hat{\mathcal{R}}_n(g_{t+1})&=\hat{\mathcal{R}}_n(g_t+\alpha\eta_{t+1}\varphi_{\boldsymbol{\theta}_{t+1}}) \nonumber \\ 
&=\hat{\mathcal{R}}_n((1-\alpha)g_t+\alpha(g_t+\eta_{t+1}\varphi_{\boldsymbol{\theta}_{t+1}})) \nonumber \\
&\leq (1-\alpha)\hat{\mathcal{R}}_n(g_t) + \alpha\hat{\mathcal{R}}_n(g_t+\eta_{t+1}\varphi_{\boldsymbol{\theta}_{t+1}}).
\end{align}
Combining Equations~\eqref{eq:ineqkappa} and \eqref{eq:ineqls} above, we obtain:
\begin{equation*}
\hat{\mathcal{R}}_n(g_{t+1}) \leq \hat{\mathcal{R}}_n(g_t)-\frac{\alpha\kappa_{t+1}^2}{2K}.
\end{equation*}
We can now use telescopic sums as in the proof of Lemma~\ref{lemma:extract}, to obtain:
\begin{equation*}
\sum_{s=0}^{S-1} \frac{\alpha\kappa_{s+1}^2}{2K} \leq \sum_{s=0}^{S-1}(\hat{\mathcal{R}}_n(g_s) - \hat{\mathcal{R}}_n(g_{s+1}))=\hat{\mathcal{R}}_n(g_0)-\hat{\mathcal{R}}_n(g_S)
\end{equation*}
for each $S\geq 1$, which leads to
\begin{equation*}
\sum_{s=1}^S \kappa_{s}^2 \leq \frac{2K}{\alpha}(\hat{\mathcal{R}}_n(g_0)-\hat{\mathcal{R}}_n(g_S)) < \infty.
\end{equation*}
As $\kappa_s \geq 0$ for each $s \geq 1$ (see the proof of Lemma~\ref{lemma:extract}), this implies that
\begin{equation}\label{eq:limkappa}
\lim_{s\rightarrow \infty} \kappa_s=0.
\end{equation}
We may now make use of Assumption~\ref{assum:compact}. By the latter and the Bolzano-Weierstrass theorem, there exists an extraction function $\sigma:\mathbb{N}\rightarrow\mathbb{N}$ such that for each $i\in\{1,\ldots,n\}$, the sequence $(g_{\sigma(s)}(\boldsymbol{x}_i))_{s\in\mathbb{N}}$ is convergent. We denote by $g_\infty$ any element of $\Span(\varphi_{\boldsymbol{\theta}})$ such that $g_\infty(\boldsymbol{x}_i)=\lim_{s\rightarrow\infty}g_{\sigma(s)}(\boldsymbol{x}_i)$ (where the dependence on the extraction function is implicit). Such an element exists because the set $\Gamma$ (defined in Assumption~\ref{assum:compact}) is closed. Let now $\boldsymbol{\theta}$ be any parameter value in $\Theta$. After iteration $\sigma(t)$, because $\boldsymbol{\theta}_{\sigma(t)+1}$ is chosen according to Equation~\eqref{eq:paramit}, we have:
\begin{equation*}
\sum_{i=1}^n\langle \nabla_2L_i(g_{\sigma(t)}),\varphi_{\boldsymbol{\theta}}(\boldsymbol{x}_i) \rangle \geq -\kappa_{\sigma(t)+1}.
\end{equation*}
And because of Assumption~\ref{assum:sym}, we also have:
\begin{equation*}
\sum_{i=1}^n\langle \nabla_2L_i(g_{\sigma(t)}),\varphi_{\boldsymbol{\theta}}(\boldsymbol{x}_i) \rangle \leq \kappa_{\sigma(t)+1}.
\end{equation*}
Now, note that by Equation~\eqref{eq:limkappa},
\begin{equation*}
\lim_{s\rightarrow\infty}\kappa_{\sigma(s)+1}=\lim_{s\rightarrow\infty}-\kappa_{\sigma(s)+1}=0.
\end{equation*}
Using the squeeze theorem, we therefore have:
\begin{equation*}
\lim_{s\rightarrow\infty}\sum_{i=1}^n\langle \nabla_2L_i(g_{\sigma(s)}),\varphi_{\boldsymbol{\theta}}(\boldsymbol{x}_i) \rangle = 0.
\end{equation*}
By the K-smoothness of $L$ in its second argument, we finally get:
\begin{equation*}
\sum_{i=1}^n\langle \nabla_2L_i(g_\infty),\varphi_{\boldsymbol{\theta}}(\boldsymbol{x}_i) \rangle = 0.
\end{equation*}
This also holds for any linear combinations of base models because of the linearity of the inner product, i.e., for all $f\in\Span(\varphi_{\boldsymbol{\theta}})$, it holds that
\begin{equation*}
\sum_{i=1}^n\langle \nabla_2L_i(g_\infty),f(\boldsymbol{x}_i) \rangle = 0.
\end{equation*}
Now, recall that by the definition of convexity, we also have for all such functions:
\begin{equation*}
\hat{\mathcal{R}}_n(g_\infty+f) \geq \hat{\mathcal{R}}_n(g_\infty) + \sum_{i=1}^n\langle \nabla_2L_i(g_\infty),f(\boldsymbol{x}_i) \rangle.
\end{equation*}
Combining the two previous equations gives $\hat{\mathcal{R}}_n(g_\infty+f) \geq \hat{\mathcal{R}}_n(g_\infty)$ for each $f\in\Span(\varphi_{\boldsymbol{\theta}})$, hence $\hat{\mathcal{R}}_n(g_\infty)=\hat{\mathcal{R}}_n(g_\star) $. We therefore have
$ \lim_{s\rightarrow\infty}\hat{\mathcal{R}}_n(g_{\sigma(s)})=\hat{\mathcal{R}}_n(g_\star) $. As the sequence $(\hat{\mathcal{R}}_n(g_s))_{s\in\mathbb{N}}$ is lower bounded and decreasing (because of the line search), it is convergent. Hence $\lim_{s\rightarrow\infty}\hat{\mathcal{R}}_n(g_s)=\hat{\mathcal{R}}_n(g_\star)$, which ends the proof.
\end{proof}

This result can easily be extended to the case where an exact line search is performed only in some iterations of the algorithm, and the other iterations do not increase the empirical risk. Such a result is useful in proving the convergence of PGB for separable loss functions, and we introduce it as a corollary below.

\begin{corollary}\label{cor:gber2}
Suppose that the same assumptions as in Theorem~\ref{th:gber2} are fulfilled. Consider a strictly increasing integer sequence $(\nu_s)_{s\in\mathbb{N}}$ in $\mathbb{N}$. Let $\nu(\mathbb{N})$ be its image. Let $g_0$ be a constant function and consider a sequence of predictors defined as:
\begin{equation*}
g_{t+1}=\begin{cases}
    g_{t}+b_{t+1}\varphi_{\boldsymbol{\theta}_{t+1}} &\text{if } t\in \nu(\mathbb{N}) \\
    \psi_{t+1}(g_{t}) &\text{otherwise}\\
\end{cases}
\end{equation*}
for $t\geq 0$, where $\psi_{t+1}(g_{t})$ is any element of $\Span(\varphi_{\boldsymbol{\theta}})$ such that $\hat{\mathcal{R}}_n(\psi_{t+1}(g_{t})) \leq \hat{\mathcal{R}}_n(g_{t})$, and $b_{t+1}$ and $\boldsymbol{\theta}_{t+1}$ are chosen as in Theorem~\ref{th:gber2}. Then $\lim_{t\rightarrow\infty}\hat{\mathcal{R}}_n(g_t)=\hat{\mathcal{R}}_n(g_\star)$.
\end{corollary}
\begin{proof}
Let $t\geq 0$. Using the exact same reasoning as in the proof of Theorem~\ref{th:gber2}, we have:
\begin{equation*}
\hat{\mathcal{R}}_n(g_{\nu(t)+1}) \leq \hat{\mathcal{R}}_n(g_{\nu(t)}) - \frac{\alpha\kappa_{\nu(t)+1}^2}{2K}.
\end{equation*}
Additionally, for $t\geq1$, because of the requirement of a non-increasing risk, we have $\hat{\mathcal{R}}_n(g_{\nu(t)}) \leq \hat{\mathcal{R}}_n(g_{\nu(t-1)+1})$, so that
\begin{equation*}
\hat{\mathcal{R}}_n(g_{\nu(t)+1}) \leq \hat{\mathcal{R}}_n(g_{\nu(t-1)+1}) - \frac{\alpha\kappa_{\nu(t)+1}^2}{2K}.
\end{equation*}
Using telescopic sums, we obtain:
\begin{equation*}
\sum_{s=1}^{S} \frac{\alpha\kappa_{\nu(s)+1}^2}{2K} \leq \sum_{s=1}^{S}(\hat{\mathcal{R}}_n(g_{\nu(s-1)+1}) - \hat{\mathcal{R}}_n(g_{\nu(s)+1}))=\hat{\mathcal{R}}_n(g_{\nu(0)+1})-\hat{\mathcal{R}}_n(g_{\nu(S)+1})
\end{equation*}
for each $S\geq1$, hence
\begin{equation*}
\lim_{s\rightarrow \infty} \kappa_{\nu(s)+1}=0.
\end{equation*}
Analogously to the proof of Theorem~\ref{th:gber2}, we can use the Bolzano-Weierstrass theorem and the boundedness of the sequence $(g_{\nu(s)})_{s\in\mathbb{N}}$ to prove the existence of an extraction function $\sigma:\mathbb{N}\rightarrow\mathbb{N}$ such that for each $i\in\{1,\ldots,n\}$, the sequence $(g_{\nu(\sigma(s))}(\boldsymbol{x}_i))_{s\in\mathbb{N}}$ converges. We also denote the corresponding limit by $g_\infty$. By the exact same reasoning as in the proof of Theorem~\ref{th:gber2}, we have that $\hat{\mathcal{R}}_n(g_\infty)=\hat{\mathcal{R}}_n(g_\star)$. As the sequence $(\hat{\mathcal{R}}_n(g_s))_{s\in\mathbb{N}}$ is lower bounded and decreasing (because of both the line search and the non-increase requirement for intermediate steps), it is convergent. We therefore have $\lim_{s\rightarrow\infty}\hat{\mathcal{R}}_n(g_s)=\hat{\mathcal{R}}_n(g_\star)$, which is the desired result.
\end{proof}

\subsection{Proof of Theorem~\ref{th:lspgb}}
The proof is a straightforward application of Corollary~\ref{cor:gber2} to the decomposition of $L$ in elementary loss functions $\ell_m$. It relies on an important property of the linear span of the family $(\varphi_{\boldsymbol{\theta}})_{\boldsymbol{\theta}\in\Theta}$ defined in Equation~\eqref{eq:wbfamily}, which is given below.
\begin{lemma}\label{lemma:span}
Let $\Span(h_{\boldsymbol{w}})$ denote the linear span of the set of base models in $(h_{\boldsymbol{w}})_{\boldsymbol{w}\in\mathcal{W}}$, and let
\begin{equation*}
\Lambda=\{g:\mathcal{X}\rightarrow\mathbb{R}^M \mid \forall m\in\{1,\ldots,M\}, g_m\in\Span(h_{\boldsymbol{w}})\}
\end{equation*}
where $g_m$ denotes the $m$-th coordinate function of $g$. Then $\Span(\varphi_{\boldsymbol{\theta}}) = \Lambda.$
\end{lemma}
\begin{proof}
We first show that $\Span(\varphi_{\boldsymbol{\theta}}) \subset \Lambda$. Let $g\in\Span(\varphi_{\boldsymbol{\theta}})$. We can write
\begin{equation*}
g=\sum_{t=0}^T b_t\boldsymbol{\beta}_t h_{\boldsymbol{w}_{t}}
\end{equation*}
where each $b_t$ lies in $\mathbb{R}$, each $\boldsymbol{\beta}_t$ in $\mathbb{S}^{M-1}$ and each $\boldsymbol{w}_t$ in $\mathcal{W}$. Then for each $m\in\{1,\ldots,M\}$, we have:
\begin{equation*}
g_m=\sum_{t=0}^T b_t\beta_{t,m} h_{\boldsymbol{w}_{t}} \in \Span(h_{\boldsymbol{w}})
\end{equation*}
where $\beta_{t,m}$ denotes the $m$-th component of $\boldsymbol{\beta}_t$. Hence $\Span(\varphi_{\boldsymbol{\theta}}) \subset \Lambda$. Let us now prove the converse relationship. Let now $g\in\Lambda$. Then
\begin{equation*}
g=\sum_{m=1}^M\mathbf{e}_mg_m
\end{equation*}
where $\mathbf{e}_m $ is the $m$-th element of the canonical basis of $\mathbb{R}^M$ (see Equation~\eqref{eq:basis} in Section~\ref{sec:pgb}). Writing $g_m=\sum_{t=0}^{T_m}b_{(t,m)}h_{\boldsymbol{w}_{(t,m)}}$ for each $m$, we get
\begin{equation*}
g=\sum_{m=1}^M\mathbf{e}_m\sum_{t=0}^{T_m}b_{(t,m)}h_{\boldsymbol{w}_{(t,m)}}=\sum_{m=1}^M\sum_{t=0}^{T_m}b_{(t,m)}\mathbf{e}_mh_{\boldsymbol{w}_{(t,m)}}\in\Span(\varphi_{\boldsymbol{\theta}}).
\end{equation*}
Hence $\Lambda \subset \Span(\varphi_{\boldsymbol{\theta}})$. As $\Span(\varphi_{\boldsymbol{\theta}}) \subset \Lambda$ and $\Lambda \subset \Span(\varphi_{\boldsymbol{\theta}})$, $\Span(\varphi_{\boldsymbol{\theta}}) = \Lambda$.
\end{proof}

\begin{proof}[Proof of Theorem~\ref{th:lspgb}]
For $m\in\{1,\ldots,M\}$, consider the sequence $(\nu_s)_{s\in\mathbb{N}^*}$ defined by $\nu_s=\rho_m(s)$ for each $s\geq1$. Let $t\geq1$. If $t\in\nu(\mathbb{N})$, we have $g_{t,m}=g_{t-1,m}+\alpha h_{\boldsymbol{w}_{t}}\gamma_{t,m}$, where $g_{t,m}$ denotes the $m$-th component of $g_t$ and with $\gamma_{t,m}$ being such that:
\begin{equation*}
\gamma_{t,m} \in \argmin_{\gamma\in\mathbb{R}}\sum_{i=1}^n\ell_m(y_i, g_{t-1,m}(\boldsymbol{x}_i)+h_{\boldsymbol{w}_{t}}(\boldsymbol{x}_i)\gamma).
\end{equation*}
Otherwise, if $t\notin\nu(\mathbb{N})$, we still have 
\begin{equation*}
\sum_{i=1}^n\ell_m(y_i, g_{t,m}(\boldsymbol{x}_i)) \leq \sum_{i=1}^n\ell_m(y_i, g_{t-1,m}(\boldsymbol{x}_i))
\end{equation*}
because of the line search. We can therefore apply Corollary~\ref{cor:gber2} to the sequence of single-output predictors $(g_{t,m})_{t\in\mathbb{N}}$, to obtain:
\begin{equation*}
\lim_{t\rightarrow\infty} \sum_{i=1}^n\ell_m(y_i, g_{t,m}(\boldsymbol{x}_i)) = \min_{f_m\in\Span(h_{\boldsymbol{w}})}\sum_{i=1}^n\ell_m(y_i, f_m(\boldsymbol{x}_i)).
\end{equation*}
As this is true for all $m$'s, we have:
\begin{equation}\label{eq:limcomp}
\lim_{t\rightarrow\infty} \hat{\mathcal{R}}_n(g_t) = \sum_{m=1}^M\min_{f_m\in\Span(h_{\boldsymbol{w}})}\sum_{i=1}^n\ell_m(y_i, f_m(\boldsymbol{x}_i)).
\end{equation}
Now, remark that, as $\Span(\varphi_{\boldsymbol{\theta}})=\Lambda$ (by Lemma~\ref{lemma:span}), we have
\begin{equation}\label{eq:riskbound}
\hat{\mathcal{R}}_n(g_\star) = \min_{g\in\Lambda}\hat{\mathcal{R}}_n(g) =\sum_{m=1}^M\min_{f_m\in\Span(h_{\boldsymbol{w}})}\sum_{i=1}^n\ell_m(y_i, f_m(\boldsymbol{x}_i)).
\end{equation}
Combining Equations~\eqref{eq:limcomp} and \eqref{eq:riskbound} results in:
\begin{equation*}
\lim_{t\rightarrow\infty} \hat{\mathcal{R}}_n(g_t) = \hat{\mathcal{R}}_n(g_\star).
\end{equation*}
\end{proof}

\section{Strategy for choosing the projection direction}\label{sec:strategy}

In this section, we compare several strategies for choosing the projection directions $(\tilde{\boldsymbol{\beta}}_t)_{t\geq1}$ at each iteration of PGB. In our implementation of PGB, we propose a simple strategy by sampling coordinates from a uniform distribution. This has some theoretical advantages, such as requiring only one $n$-tuple of pseudo-residuals to be maintained at each iteration, instead of the entire gradient matrix. However, other strategies can be imagined, including (but not limited to) sampling the directions from other distributions. If the adopted strategy is adaptive to the current state of the algorithm, one may hope that this improves the dynamics of the descent and leads to better performance, either in terms of computation time or quality of the predictions.

Although PGB had not been studied as such previously, there are a few published works containing edge cases that fit within the PGB framework, namely with decision trees as base learners \parencite{joly2019gradient, iosipoi2022sketchboost}. They come with specific, adaptive and non-adaptive strategies for choosing the projection directions. Proving the theoretical superiority of one such strategy over the others is difficult. Even when one restricts the set of candidate directions to the canonical basis of $\mathbb{R}^M$, the problem inherits the difficulty of block selection in block coordinate gradient descent (\cite{nutini2022let}; recall that gradient boosting can be seen as gradient descent in a functional space), and is further complicated by the involvement of a potentially nonlinear class of base learners. Instead of seeking elaborate theoretical arguments, we therefore limit ourselves to an empirical comparison of such strategies. We consider the following candidate strategies, some of which are inspired by the works mentioned above \parencite{joly2019gradient, iosipoi2022sketchboost}:
\begin{itemize}
\item Uniform sampling on the canonical basis (\texttt{Unif-basis}). This is the strategy chosen in our implementation of PGB, which is presented in more detail in Section~\ref{sec:pgb}.
\item Uniform sampling on the sphere (\texttt{Unif-sphere}). Here, the direction $\tilde{\boldsymbol{\beta}}_t$ is not restricted to belong to the canonical basis of $\mathbb{R}^M$. Rather, it is sampled from the uniform measure on $\mathbb{S}^{M-1}$. For that, one draws an $M$-tuple $\boldsymbol{z}=(z_1,\ldots, z_M)$ from a standard multivariate Gaussian distribution. The direction $\tilde{\boldsymbol{\beta}}_t$ is then taken as $\frac{\boldsymbol{z}}{\Vert \boldsymbol{z}\Vert_2}$ \parencite{muller1959note}.
\item Adaptive sampling on the canonical basis (\texttt{Adapt-basis}). In this variant of the \texttt{Unif-basis} strategy, vectors of the canonical basis of $\mathbb{R}^M$ are not sampled from a uniform distribution, but with probabilities proportional to the norm of the corresponding column of the gradient matrix. More specifically, the probability $p_m$ of sampling $\mathbf{e}_m$ for $m\in\{1,\ldots,M\}$ is such that:
\begin{equation*}
    p_m\propto \sqrt{\sum_{i=1}^n \langle\nabla_2L_i(g_{t-1}), \mathbf{e}_m\rangle^2}
\end{equation*}
where the actual probabilities used are normalized to sum to 1.
\item Gauss-Southwell-type strategy (\texttt{Gauss-Southwell}). Analogously to the well-known rule for coordinate gradient descent \parencite{nutini2015coordinate}, this strategy takes as direction the vector of pseudo-residuals with the greatest Euclidean norm, i.e., $\tilde{\boldsymbol{\beta}}_t=\nabla_2L_{i_{\operatorname{max}}}(g_{t-1})$, where
\begin{equation*}
    i_{\operatorname{max}}\in\argmax_{1\leq i\leq n} \Vert \nabla_2L_i(g_{t-1}) \Vert_2.
\end{equation*}
Note that another option could be to choose the component $m$ in $\{1,\ldots,M\}$ yielding pseudo-residuals with the greatest norm, i.e.,
\begin{equation*}
    \tilde{\boldsymbol{\beta}}_t \in \argmax_{\mathbf{e}_m \in \{\mathbf{e}_1,\ldots,\mathbf{e}_M\}} \sum_{i=1}^n \langle\nabla_2L_i(g_{t-1}), \mathbf{e}_m\rangle^2.
\end{equation*}
However, as noted by \textcite{iosipoi2022sketchboost}, this does not account for the case where the components $\ell_m$ of the loss naturally result in gradients of different magnitudes, such as in multiple quantile regression. We therefore do not consider this strategy here.
\item Greedy selection (\texttt{Greedy}). We include this approach mainly to serve as a reference for the other strategies, as it requires training $M$ base learners per iteration and therefore defeats the main purpose of PGB. As its name indicates, it is a greedy strategy in which the direction at iteration $t$ is selected from the canonical basis of $\mathbb{R}^M$ to maximize the decrease in empirical risk, that is,
\begin{equation*}
    \tilde{\boldsymbol{\beta}}_t \in \argmin_{\mathbf{e}_m \in \{\mathbf{e}_1,\ldots,\mathbf{e}_M\}} \hat{\mathcal{R}}_n(g_t)
\end{equation*}
where the dependence of $g_t$ on the candidate directions $\mathbf{e}_1,\ldots,\mathbf{e}_M$, as in Equations~\eqref{eq:wttilde1} and \eqref{eq:wttilde2}, is implicit. The search is limited to the canonical basis of $\mathbb{R}^M$ for computational reasons (the problem being, in general, intractable on the whole hypersphere).
\end{itemize}

To compare these strategies, we perform the following experiment with each of the datasets described in Table~\ref{tab:data}. Each of the datasets is randomly split in half to form a training and test set of observations. We fit a PGB model to estimate 50 regularly spaced conditional quantiles (see Section~\ref{sec:mqr}) based on the training data, using each of the strategies mentioned above. In each case, the base learners used are trees, with the same hyperparameters as in Section~\ref{sec:experiments}, except for the number of trees, which is set to 100, and the learning rate, which is chosen by 5-fold cross-validation within the set $\{0.01, 0.025, 0.05, 0.1, 0.2, 0.5\}$. This is also done with an XGBoost model for comparison purposes, with the same hyperparameters and tuning strategy. The experiment is run 50 times for each dataset to provide estimates of the generalization errors. The results of this experiment are given in Figure~\ref{fig:strategy}.

\begin{figure}[!htbp]
\centering
\includegraphics[width=0.8\linewidth]{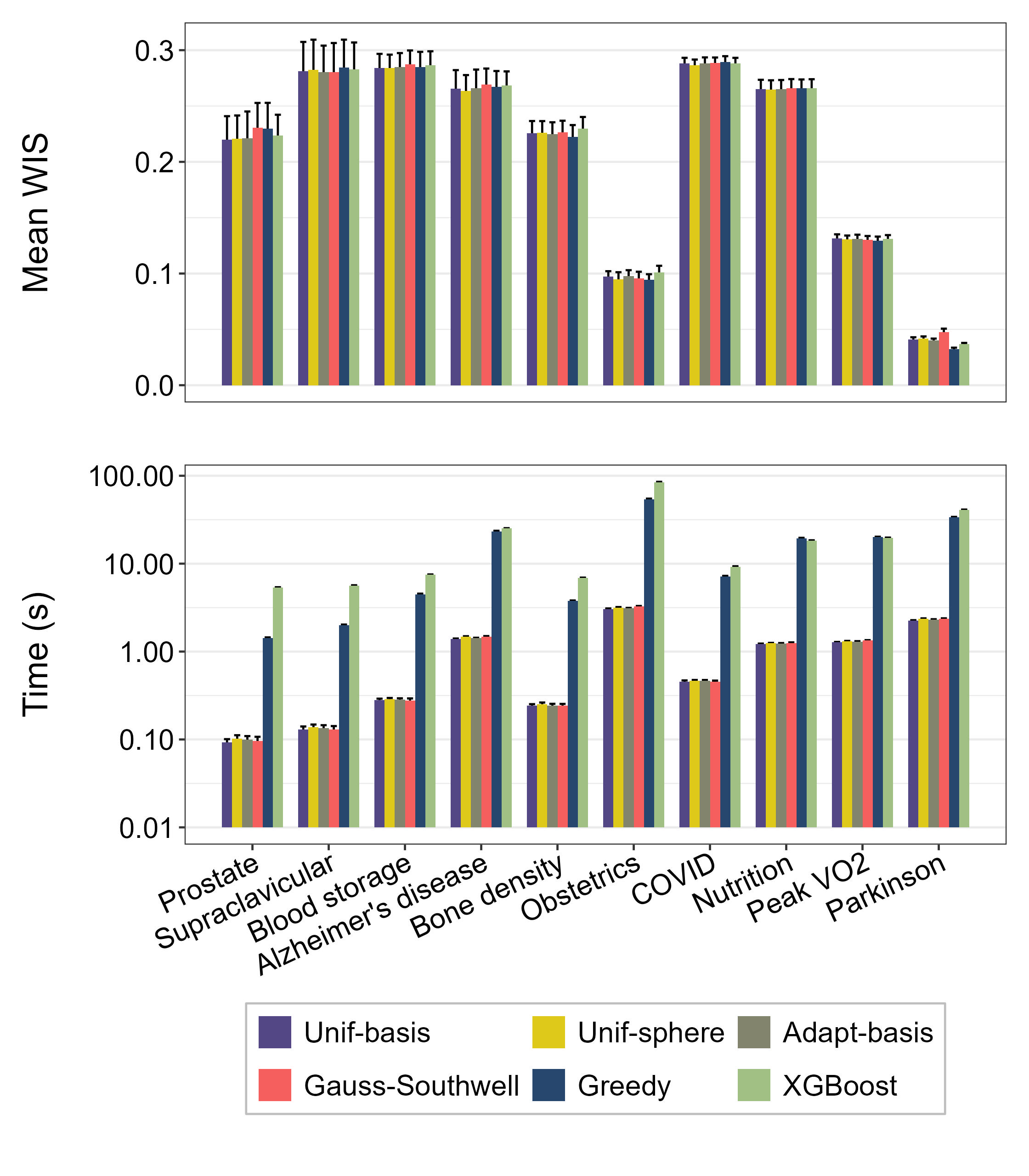}
\caption{\label{fig:strategy}Performance of various selection strategies for the projection directions in parallel gradient boosting, applied to real medical datasets. The error bars represent one standard deviation. Note the logarithmic scale for the bottom plot.}
\end{figure}

We see that all strategies yield comparable test errors, irrespective of the dataset. The errors are also close to those obtained with XGBoost (which supports the results of a similar experiment with the baseball dataset, shown in Figure~\ref{fig:crossing}). In terms of computation times, the \texttt{Unif-basis} strategy consistently appears as one of the fastest methods, though with only minimal differences compared to the other non-greedy strategies. Compared with the \texttt{Greedy} strategy and XGBoost, the gains are, as expected, massive. These results therefore support the choice of the \texttt{Unif-basis} strategy as a default method in our implementation of PGB.

\section{Convergence of the algorithm with linear base learners}\label{sec:linear}

Contrary to the existing literature in which PGB appears as an edge case, and is developed only as a method for growing decision trees \parencite{joly2019gradient, iosipoi2022sketchboost}, our formulation of PGB accounts for general classes of base learners. To illustrate how PGB behaves with base learners other than trees, we run the following experiment. On each dataset described in Table~\ref{tab:data}, we fit a PGB model to minimize the following loss function:
\begin{equation*}
    L(y_i, g(\boldsymbol{x}_i))=\sum_{m=1}^M(y_i^m-g_m(\boldsymbol{x}_i))^2
\end{equation*}
where $g_m(\boldsymbol{x}_i)$ denotes the $m$-th component of $g(\boldsymbol{x}_i)$. The oracle for this loss is the tuple of $M$ conditional moments of $Y$ given $X$. The class of base learners used is the class of linear models to predict $Y$, combined in $500$ iterations with a shrinkage parameter of $0.01$. To stabilize the optimization, we standardize the regression targets $(y_1^m,\ldots,y_n^m)$ for $1\leq m \leq M$ by their empirical mean and variance. We report the evolution of the empirical risk at each iteration in Figure~\ref{fig:linear}.

\begin{figure}[!htbp]
\centering
\includegraphics[width=1\linewidth]{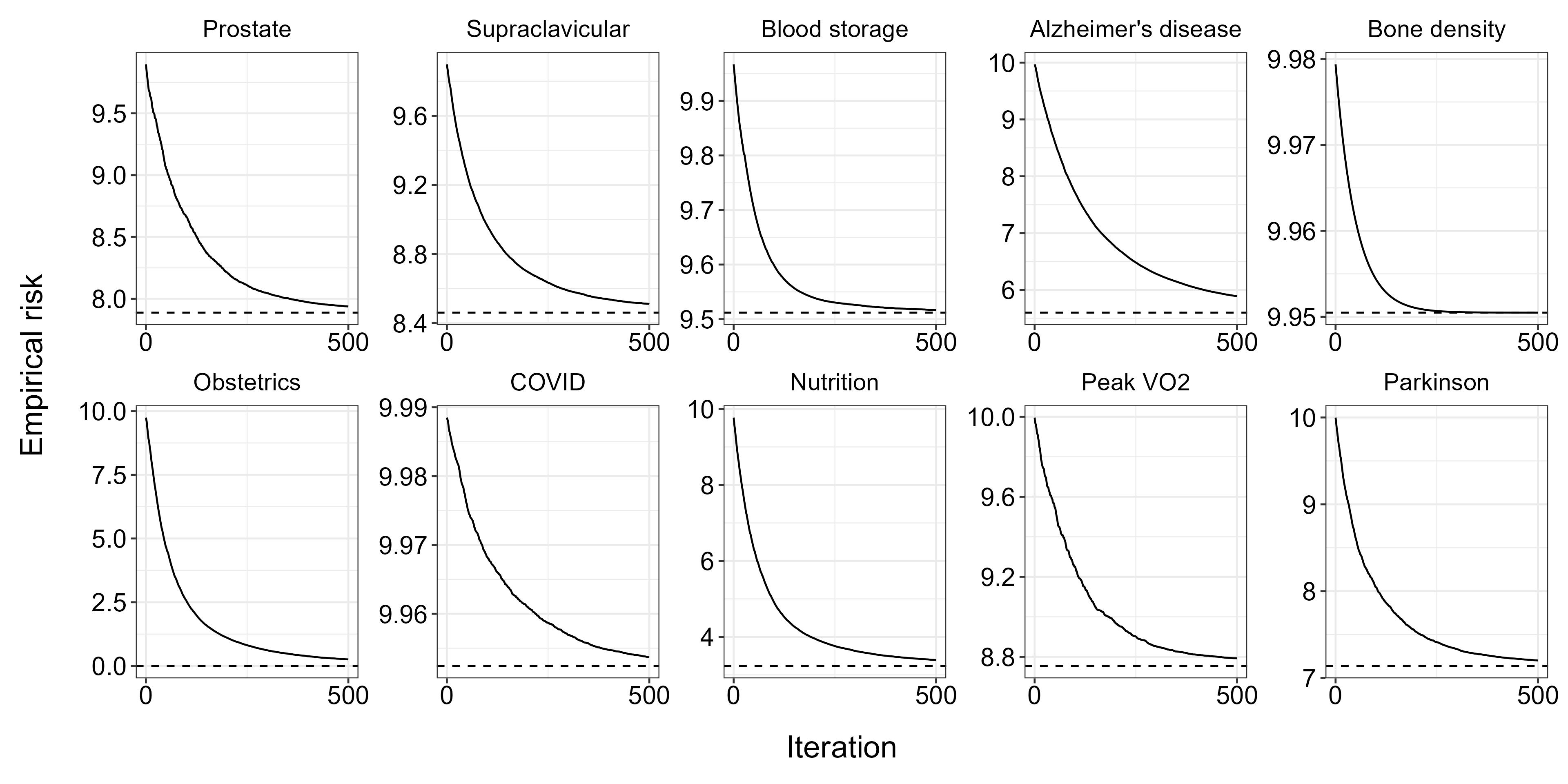}
\caption{\label{fig:linear}Training curves for the 10 real datasets to predict the $M$-tuple of conditional moments with linear base learners. The dashed horizontal line represents the optimal risk attainable in this class of models.}
\end{figure}

For all 10 datasets, the algorithm appears to converge to the minimum empirical risk attainable with this class of base learners. This illustrates the theoretical results stated in Section~\ref{sec:pgb}, especially their applicability to base learners other than the most traditional decision trees.

\section{Effect of the number of quantiles on the performance of the estimator}\label{sec:ntargets}

To investigate how the number of quantile knots affects the performance of PGB in estimating conditional distributions, we simulate 2000 observations from the following model:
\begin{equation*}
R\sim\mathcal{U}(\{1,\ldots,k\}), \quad
A\sim\mathcal{U}(\interval{0}{2\pi}), \quad
\epsilon\sim\mathcal{N}(0,0.0025),
\end{equation*}
\begin{equation*}
X=(0.5R+\epsilon)\cos(A), \quad
Y=(0.5R+\epsilon)\sin(A)
\end{equation*}
where $k$ is a natural integer controlling the number of modes in the conditional distribution of $Y$ given $X$. Observations in this model are located near $k$ concentric circles of increasing radius, with $\epsilon$ acting as a Gaussian radial perturbation around these circles. Conditionally on $X$, the distribution of $Y$ therefore has up to $2k$ modes. We run the PGB algorithm to estimate these conditional distributions, with an increasing number $M$ of target quantiles, and other hyperparameters set as in Section~\ref{sec:experiments}. The quality of the estimation is measured by the ISE, approximated as in Section~\ref{sec:experiments} using 2000 other observations sampled from the same distribution. We conduct the experiment successively with $k=1$, $2$ and $3$. The results are given in Figure~\ref{fig:modes}. 

\begin{figure}[!htbp]
\centering
\includegraphics[width=0.9\linewidth]{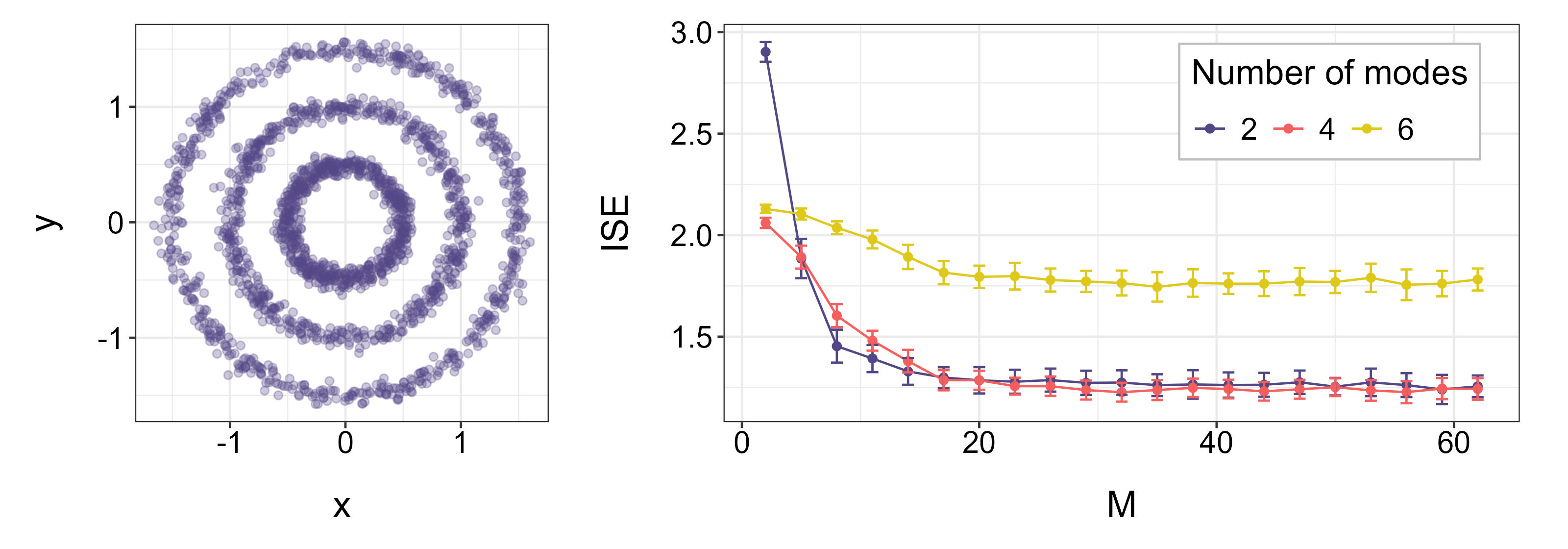}
\caption{\label{fig:modes}Right: Quality of the estimations of conditional distributions provided by parallel gradient boosting with different grids of target quantiles, on three datasets with an increasing maximum number of conditional modes. Left: plot of the dataset obtained with $k=3$. The error bars represent one standard deviation.}
\end{figure}

In this experiment, increasing the number of target quantiles consistently improves the quality of the estimation (most likely through reduction of the approximation error, as discussed in Section~\ref{sec:cqtocd}). The improvement, however, becomes insignificant beyond $M=20$, regardless of the complexity of the conditional distribution of $Y$ given $X$ (as indicated by its maximum number of modes). We conclude that the hyperparameter $M$ seems to play no role in the regularization strategy of our implementation of PGB, but instead controls the balance between computation times and the quality of the approximation of the CRPS.

\end{appendices}

\newpage

\section*{References}

\printbibliography[heading=none]

\end{document}